\documentclass{article}

\usepackage{times}
\usepackage[letterpaper,textwidth=5.5in,textheight=9in]{geometry}
\usepackage[round]{natbib}
\usepackage[utf8]{inputenc}
\usepackage[T1]{fontenc}
\usepackage{amsmath,amssymb,bm,mathtools}
\usepackage{booktabs,multirow,tabularx,array}
\usepackage{graphicx}
\usepackage{microtype}
\usepackage{xcolor}
\usepackage{enumitem}
\usepackage{float}
\usepackage{wrapfig}
\usepackage{capt-of}
\usepackage{amsthm}
\usepackage{tikz}
\usetikzlibrary{arrows.meta,positioning,fit,backgrounds,calc}
\usepackage{hyperref}
\hypersetup{colorlinks=true,linkcolor=red!60!black,citecolor=blue!60!black,urlcolor=blue!60!black}
\usepackage{url}

\graphicspath{{figures/}}
\newcommand{\ECROM}{\textsc{Ecrom}}
\newcommand{\Prob}{\mathbb{P}}
\newcommand{\E}{\mathbb{E}}
\newcommand{\Ber}{\operatorname{Bernoulli}}
\newcommand{\BetaDist}{\operatorname{Beta}}
\newcommand{\ind}{\mathbf{1}}
\newcommand{\supportset}{\mathcal{J}}
\newtheorem{proposition}{Proposition}


\title{\vspace{-2.5em}Retrospective Open-Vocabulary Memory for Long-Term Object Search}

\author{Jiaming Wang \quad Zhiwei Xue \quad Chen Jizhuo \quad Peng Shiqi \quad Harold Soh\\[2pt]
National University of Singapore}
\date{\vspace{-1.5em}}

\begin{document}
\maketitle

\begin{abstract}
Long-term object search requires learning where objects \emph{usually} appear from repeated but uneven observations of a changing environment. We formulate retrospective open-vocabulary memory as probabilistic inference from censored observations, where the key idea is to reason with \emph{evidence per opportunity}: a detection or non-detection should influence belief only in proportion to the robot's opportunity to observe the corresponding location. We introduce \ECROM{}, which uses this principle to estimate long-term prevalence for concepts specified only at query time and converts the resulting belief directly into an active-search prior. To evaluate this problem, we introduce a controlled long-term benchmark in ten HM3D homes that independently varies object placement and observation opportunity across repeated traversals. \ECROM{} improves support-level AP on held-out queries by 4.5 points and search SPL by 4.2 points over the strongest competing memory in each metric. The benchmark, dataset, and code will be open-sourced. Project page: \url{https://jiaming.im/ecrom/}.
\end{abstract}

\begin{figure}[H]
\centering
\includegraphics[width=0.9\linewidth]{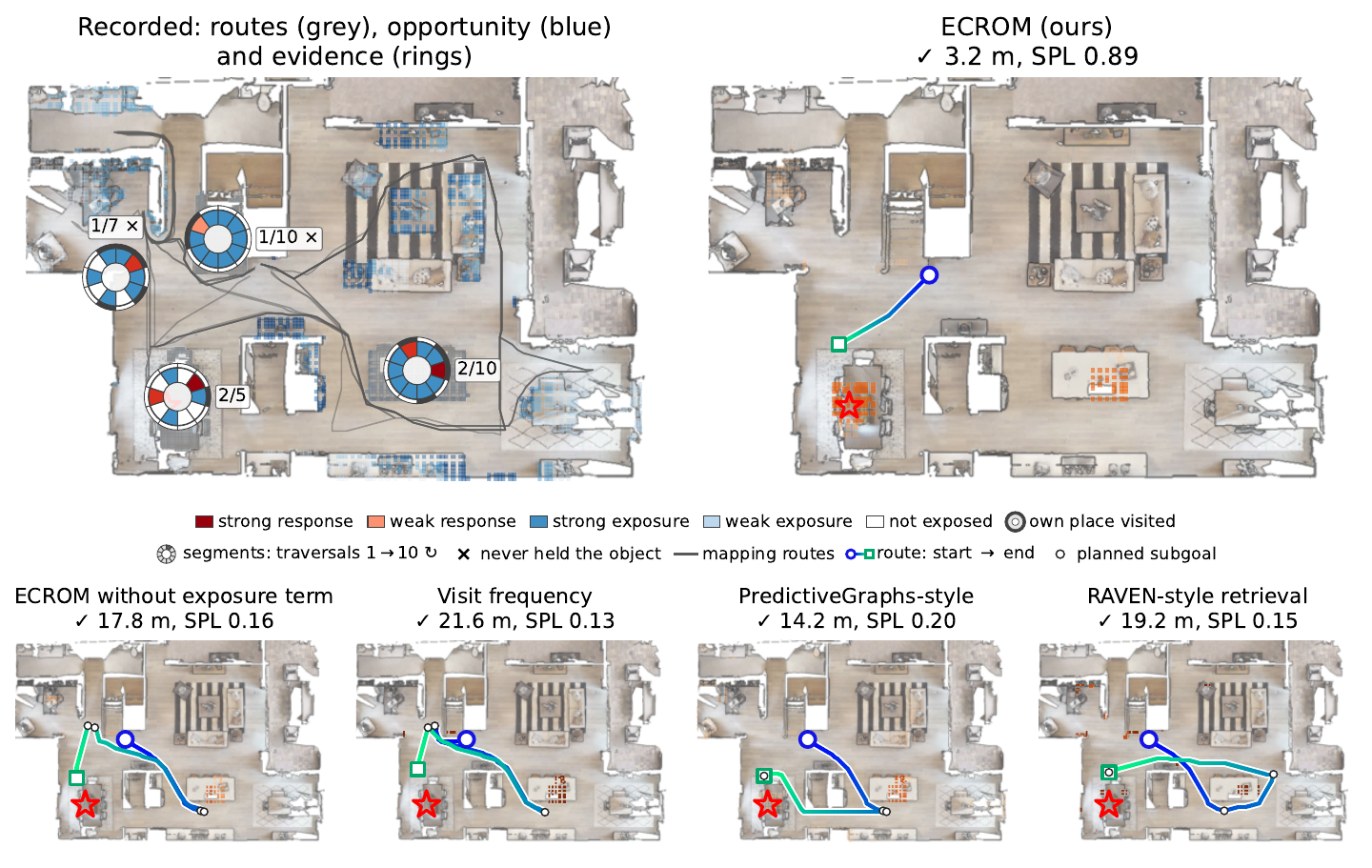}
\caption{\small{Searching for ``a straw hat'' in an HM3D home after ten mapping traversals. Two candidate surfaces, the dining table and the kitchen island, have the same number of detector responses but different observation opportunity. \emph{Top left}: mapping routes (grey), mean observation opportunity of each surface patch (blue), and one ring per candidate surface with one segment per traversal, colored by detector response and exposure; a label $k/n$ gives $k$ responses in the $n$ traversals that exposed the surface (table 2/5, island 2/10), and $\times$ marks surfaces that never held the hat. \emph{Other panels}: each memory's belief and the route induced from the same start under a 120\,m budget. \ECROM{} uses censored observations to exploit informative doorway views and verifies the table after only 3.2\,m. \ECROM{} without exposure term and visit frequency ignore how much of each surface was seen and favor the island, reaching the hat after 17.8\,m and 21.6\,m. The PredictiveGraphs-style persistence filter also favors the island (14.2\,m), while RAVEN-style retrieval ranks stored views and revisits the island from two sides before reaching the table (19.2\,m). Coverage (no memory) requires 52.8\,m. }}

\label{fig:teaser}
\end{figure}

\section{Introduction}

Object search is a core capability of embodied AI. Existing approaches learn category-level search priors for unseen environments~\citep{chaplot2020semexp,ramakrishnan2022poni}, ground open-vocabulary queries in spatial maps~\citep{huang2023vlmaps,gu2024conceptgraphs}, or reuse a map across multiple goals~\citep{khanna2024goat}. These settings largely assume that objects are static or that the latest observation represents their current location.

Real environments are different. Frequently used objects such as keys, tools, and mugs move repeatedly, yet often follow stable routines. A robot deployed over weeks or months therefore should not only remember where an object was last seen, but learn \emph{where it usually occurs}. Such a memory can guide future search and support retrospective open-vocabulary queries even for concepts specified only after the observations were recorded.

The key challenge is that repeated observations are highly uneven. Visiting a room does not imply observing every surface within it: the camera may face elsewhere, a surface may be occluded, or only part of it may be visible. Consequently, detection counts or visit frequencies confound object occurrence with observation frequency. A detection or non-detection should instead affect belief according to the robot's \emph{opportunity to observe} the corresponding location. Figure~\ref{fig:teaser} illustrates this distinction: two surfaces can produce the same number of detections yet imply very different long-term prevalence because they were observed with different frequency.

Existing long-term memories primarily retrieve historical observations~\citep{hu2026raven}, maintain the latest scene state~\citep{liu2025dynamem}, represent temporal scene histories~\citep{gorlo2026daaam}, or predict the future state of established object--receptacle relations~\citep{saavedraruiz2026predictivegraphs}. Our goal is different: to estimate the long-term prevalence of an arbitrarily specified concept while explicitly accounting for whether each past episode provided a meaningful chance to observe each location.

We introduce \ECROM{} (\emph{Exposure-Calibrated Retrospective Open-Vocabulary Memory}), a query-independent long-term visual memory based on \emph{evidence per opportunity}. Repeated RGB-D traversals are organized into a sparse atlas of surface supports. For each support $j$ and episode $e$, the memory stores its geometric observation opportunity $c_{je}\in[0,1]$. At query time, open-vocabulary detector responses are converted into calibrated likelihood-ratio evidence and combined with these opportunities through a censored observation model to infer the long-term prevalence $\theta_j^q$ of concept $q$. The resulting belief is used directly as a prior for active search.

Our contributions are:
\begin{itemize}[leftmargin=1.4em,itemsep=1pt,topsep=2pt]
    \item a formulation of retrospective open-vocabulary memory as censored inference over long-term object prevalence, explicitly separating semantic evidence from observation opportunity;
    \item \ECROM{}, a sparse query-independent memory that estimates prevalence for concepts specified only at query time and converts it directly into an active-search prior;
    \item LOOP-Bench, a controlled long-term benchmark in ten realistic homes where object occurrence and observation opportunity vary independently; and
    \item experiments showing improved distribution estimation and active search over counting, retrieval-, latest-state-, and PredictiveGraphs-style memories, including deployment on a physical robot.
\end{itemize}

\section{Related Work}

\paragraph{Object search and navigation.}
Object-goal navigation and semantic exploration learn category-level priors or object probability maps for searching unseen environments~\citep{chaplot2020semexp,ramakrishnan2022poni,wang2023smallobj,wang2024polo}, while open-vocabulary maps and 3D scene graphs such as VLMaps~\citep{huang2023vlmaps} and ConceptGraphs~\citep{gu2024conceptgraphs} ground language queries in a spatial map. These approaches primarily reason about \emph{where an object is}, whereas we use repeated observations of a changing environment to estimate \emph{where it usually occurs}.

\paragraph{Spatio-temporal memory.}
Recent memories explicitly represent changes over time. DynaMem~\citep{liu2025dynamem} and DovSG~\citep{yan2025dovsg} update an open-vocabulary 3D map or scene graph to the latest observed state; DAAAM represents historical states in a hierarchical 4D scene graph~\citep{gorlo2026daaam}; and retrieval-based approaches such as RAVEN~\citep{hu2026raven}, STAR~\citep{chen2026star}, and ObsGraph~\citep{lee2026obsgraph} preserve and retrieve informative past observations. FreMEn~\citep{krajnik2017fremen} learns periodic patterns in the probability of predefined binary states, such as whether a map cell is occupied, and PredictiveGraphs~\citep{saavedraruiz2026predictivegraphs} goes further by attaching probabilistic temporal filters to object--receptacle relations, modeling persistence, emergence, detector misses, and missing observations to predict future relation states. These memories retrieve past observations, maintain the current state, or predict how predefined states and established relations evolve; none estimates the long-term prevalence of a concept specified after the fact while accounting for whether each episode could observe each location.

\paragraph{Benchmarks for object search and scene change.}
ObjectNav~\citep{batra2020objectnav}, including its version in procedurally generated ProcTHOR houses~\citep{deitke2022procthor}, evaluates object search in static, single-episode environments, while GOAT-Bench~\citep{khanna2024goat} reuses a map across multiple goals but does not move objects between them. RIO~\citep{wald2019rio} records repeated scans of changing indoor scenes to evaluate re-localizing objects that moved, not learning where they usually occur. LOOP-Bench instead repeatedly samples object locations from hidden distributions while independently varying robot trajectories and viewpoints, creating uneven observation opportunity. Queries are specified only after the history is recorded, and evaluation measures both recovery of the long-term occurrence distribution and its use for active search.

\section{Problem Setup}
\label{sec:problem}

We record episodes (traversals) $e\in\{1,\ldots,E\}$, each containing a stream of RGB-D frames, odometry, and timestamps. Our goal is to build a query-independent memory $\mathcal M$ that can later infer where an arbitrarily specified object is likely to occur and use this prior to search efficiently, rather than exploring the environment from scratch. We organize this history as a place--view atlas built on the CROSS keyframe graph~\citep{wang2026cross}. Its sparse topological structure tolerates changes between sessions, such as moved objects, without fusing all observations into one static metric map. On the keyframe graph, we group nearby keyframes into graph-connected places $a$ and bin their viewing directions into canonical view directions $\psi$, yielding executable place--view pairs, or views, $u=(a,\psi)$. The stored depth is voxelized into surface supports $j\in\supportset$, where each support owns a set of surface voxels $\mathcal S_j$. A support is a fixed-size, class-agnostic geometric patch of up-facing surface ($0.5\times0.5$\,m per height band; Appendix~\ref{app:impl}). Thus, a view $u$ specifies \emph{where the robot can look}, while $j$ specifies \emph{where an object may occur}. After mapping, the user supplies arbitrary text $q$. Let $X_{je}^q\in\{0,1\}$ indicate whether an instance satisfying $q$ is physically present on support $j$ in episode $e$; an instance counts as present on $j$ when it lies within 45\,cm of a voxel of $\mathcal S_j$, the radius at which detections are attributed to supports. Its long-term prevalence at support $j$ is \begin{equation} \theta_j^q \triangleq \Prob(X_{je}^q=1),  \label{eq:target} \end{equation} which we assume to be stationary across episodes. The memory infers the posterior of this prevalence given the recorded history $\mathcal H_j$ of support $j$ and the query, $p(\theta_j^q\mid\mathcal H_j,q)$, independently for every support; its mean is the search prior over supports (Section~\ref{sec:planner}).

\section{Evidence per Opportunity}
\label{sec:method}

\begin{figure}[t]
\centering
\resizebox{\linewidth}{!}{\input{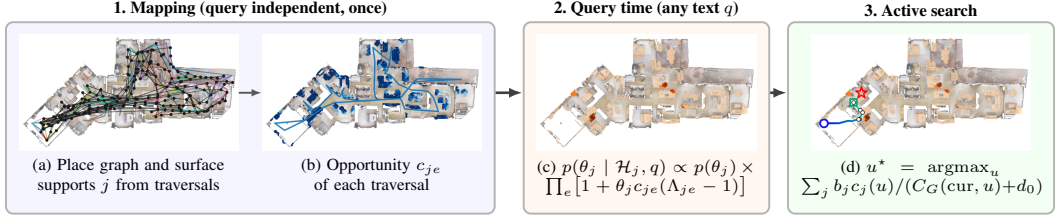}}
\caption{\small{\textbf{System overview.} (1)~During query-independent mapping, RGB-D traversals are organized into a place--view atlas with surface supports and per-episode observation opportunities $c_{je}$. (2)~At query time, open-vocabulary detector responses are lifted to supports and combined with opportunity to infer the long-term prevalence posterior $p(\theta_j\mid \mathcal H_j, q)$ (Section~\ref{sec:posterior}). (3)~The resulting belief guides budgeted active search and is updated after each executed view (Eq.~\ref{eq:update}).}}
\label{fig:overview}
\end{figure}

Each recorded episode should change the belief about a support according to two separate sources of information, \emph{semantic evidence} and \emph{observation opportunity}: the memory must distinguish a genuine non-detection from a support that was never meaningfully observed, while keeping graded detector evidence rather than a binary hit. We derive an observation model that separates the two and use the resulting posterior as the search prior (Figure~\ref{fig:overview}).

\subsection{A censored observation model} \label{sec:model}

A low detector response may mean that the object was absent, but also that the relevant surface was barely or never observed. We therefore derive the sensor likelihood by explicitly separating \emph{physical presence} from the \emph{opportunity to observe it} (Figure~\ref{fig:latent}). For support $j$ in episode $e$ (we drop the query index $q$ when it is clear), we model physical presence $X_{je}$ and whether the recorded views expose the support, $V_{je}$, as
\begin{equation} X_{je}\sim\Ber(\theta_j),\qquad V_{je}\sim\Ber(c_{je}), \end{equation} 
where $\theta_j$ is the unknown long-term prevalence and $c_{je}\in[0,1]$ is the observation opportunity measured from the stored geometry. We assume $X_{je}\perp V_{je}$: where an object is placed is independent of whether the robot happens to observe that surface during the episode. The object can affect the detector only when both events occur, so we define the latent \emph{visible presence} $Z_{je}=X_{je}V_{je}$. At query time, the frozen detector produces semantic evidence $r_{je}$ with class-conditional densities and likelihood ratio \begin{equation}
\label{eq:Lambda}
r_{je}\mid Z_{je}{=}z\;\sim\;p_z(r),\quad z\in\{0,1\}, \qquad \Lambda(r)=\frac{p_1(r)}{p_0(r)}. 
\end{equation} 

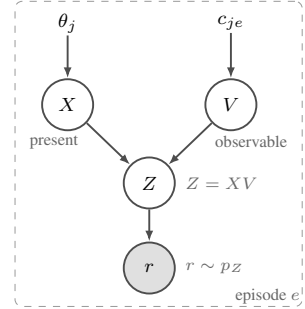
\begin{wrapfigure}[15]{r}{0.30\linewidth} \vspace{-14pt} \centering \scalebox{0.9}{
\begin{tikzpicture}[
  font=\footnotesize,
  x=1cm, y=1cm,
  lat/.style={circle, draw=black!70, thick, minimum size=7.5mm, inner sep=0pt, fill=white},
  obs/.style={lat, fill=black!12},
  par/.style={inner sep=1pt},
  arr/.style={-{Latex[length=1.6mm]}, thick, black!70},
  lab/.style={font=\scriptsize, text=black!60, inner sep=1pt}
]
  \node[par] (theta) at (0,2.35) {$\theta_j$};
  \node[par] (c) at (2.4,2.35) {$c_{je}$};
  \node[lat] (X) at (0,1.15) {$X$};
  \node[lat] (V) at (2.4,1.15) {$V$};
  \node[lat] (Z) at (1.2,0.0) {$Z$};
  \node[obs] (r) at (1.2,-1.25) {$r$};
  \draw[arr] (theta) -- (X);
  \draw[arr] (c) -- (V);
  \draw[arr] (X) -- (Z);
  \draw[arr] (V) -- (Z);
  \draw[arr] (Z) -- (r);
  \node[lab, anchor=north] (lp) at ($(X.south)+(-0.2,-0.02)$) {present};
  \node[lab, anchor=north] (lv) at ($(V.south)+(0.3,-0.02)$) {observable};
  \node[lab, anchor=west] at ($(Z.east)+(0.1,0)$) {$Z=XV$};
  \node[lab, anchor=west] at ($(r.east)+(0.1,0)$) {$r\sim p_Z$};
  \begin{scope}[on background layer]
    \node[draw=black!40, rounded corners=3pt, dashed, inner sep=5pt, fit=(X)(V)(Z)(r)(c)(lp)(lv)] (plate) {};
  \end{scope}
  \node[lab, anchor=south east] at (plate.south east) {episode $e$};
\end{tikzpicture}} 
\caption{\small{Latent model for one support and episode: presence $X$, observability $V$, visible presence $Z=XV$, detector evidence $r$.}} \label{fig:latent} \vspace{-10pt} \end{wrapfigure} 
Marginalizing the unobserved presence and observability variables gives the likelihood used by \ECROM{}. 

\begin{proposition}[Opportunity-weighted likelihood] 
\label{prop:likelihood} 
Under the model above, 
\begin{equation} p(r_{je}\mid\theta_j,c_{je}) = p_0(r_{je}) \left[ 1+\theta_jc_{je}\bigl(\Lambda(r_{je})-1\bigr) \right]. \label{eq:likelihood} \end{equation} \end{proposition} 
A derivation is given in Appendix~\ref{app:derivation}.
Equation~\ref{eq:likelihood} is the central ``evidence per opportunity'' rule. The detector determines the \emph{direction} of the update: $\Lambda>1$ supports presence, $\Lambda<1$ supports absence, and $\Lambda\approx1$ is uninformative. Opportunity determines its \emph{strength}. As $c_{je}\rightarrow0$, either kind of semantic evidence becomes neutral; in particular, a missing detection cannot penalize a surface that the robot had little chance to observe. When $c_{je}$ is large, both positive and negative evidence matter correspondingly more.

\subsection{Observation opportunity and semantic evidence}
\label{sec:opp}
\emph{Opportunity.}
For episode $e$, let $\Omega_e$ denote the set of surface voxels that returned depth within the detection range in at least one stored frame. For support $j$, we define
\begin{equation}
c_{je}=
\frac{
\sum_{\xi\in\mathcal S_j}
\ind[\xi\in\Omega_e]
}{
|\mathcal S_j|
}
\in[0,1],
\label{eq:exposure}
\end{equation}
the fraction of the support's voxels that returned depth in episode $e$ within the detection range; supports contain only up-facing voxels between $0.12$\,m and $1.5$\,m above the floor, so the floor is never a support. We use depth return as a geometric proxy for observation opportunity and also call $c_{je}$ the \emph{exposure} of support $j$ in episode $e$: a surface that returned depth was exposed to the sensor and could, in principle, have revealed an object resting on it. Occluded or unobserved regions contribute nothing. Exposure is thus a heuristic approximation of the probability that a present object is visible; Section~\ref{sec:limitations} discusses what it ignores. The opportunity is query independent and is therefore computed once during mapping.

\emph{Semantic evidence.}
Given a text query $q$, OWLv2~\citep{minderer2023owlv2} scores the stored RGB frames. We define $r_{je}$ as the maximum localized detector score assigned to support $j$ in episode $e$, avoiding repeated evidence from adjacent frames; a support with no localized detection in episode $e$ has $r_{je}=0$.

We calibrate the likelihood ratio $\Lambda(r)=p_1(r)/p_0(r)$ once and share it across all queries. In simulation, we calibrate on the designated calibration objects, which are never evaluated and for which the simulator provides the latent variable of Eq.~\ref{eq:Lambda} itself: $Z_{je}=1$ when the object is present on support $j$ in episode $e$ and visible in at least one stored keyframe of that episode (at least 50 pixels within 3\,m). We histogram the responses $r_{je}$ of all episode--support cells of the calibration objects separately for $Z_{je}=1$ and $Z_{je}=0$, pooled over houses, and use isotonic regression so that a higher score never gives weaker evidence for presence. On the real robot, we use a VLM (GPT-6) to label sampled stored keyframes as showing the object or not and fit $\Lambda$ from these labels. Further details are provided in Appendix~\ref{app:calibration}.

\subsection{Posterior over long-term prevalence}
\label{sec:posterior}

We place a Beta prior on the prevalence of each support, $\theta_j\sim\BetaDist(\alpha_0,\beta_0)$.
 Let $ \mathcal H_j = \{(r_{je},c_{je})\}_{e=1}^{E}$
denote the recorded evidence and opportunities for support $j$. Assuming episodes are conditionally independent given $\theta_j$, Bayes' rule gives the posterior below.

\begin{proposition}[Posterior prevalence]
\label{prop:posterior}
\begin{equation}
p(\theta_j\mid\mathcal H_j,q)
\propto
\theta_j^{\alpha_0-1}
(1-\theta_j)^{\beta_0-1}
\prod_{e=1}^{E}
\left[
1+\theta_jc_{je}
(\Lambda(r_{je})-1)
\right].
\label{eq:posterior}
\end{equation}
\end{proposition}

The posterior has two useful invariances. First, an episode with $c_{je}=0$ contributes a factor of one and therefore leaves the belief unchanged. Second, duplicating frames within an episode changes neither $c_{je}$, which is computed from a union of observed surface voxels, nor $r_{je}$, which is a maximum over localized scores.

The prior is sparse, $\alpha_0\ll\beta_0$, reflecting that a particular named object is unlikely to occupy any individual surface patch. The prior parameters are selected using the calibration objects and fixed for all evaluation queries. Because Eq.~\ref{eq:posterior} is generally non-conjugate, we evaluate the posterior numerically using a grid with increased resolution near $\theta_j=0$ (Appendix~\ref{app:impl}). The posterior mean
$
\hat{\theta}_j^q=\E[\theta_j\mid\mathcal H_j,q]
$
gives our estimate of the long-term prevalence of concept $q$ at support $j$.

\subsection{From memory to active search}
\label{sec:planner}

For a unique target, we normalize the posterior prevalence estimates into search weights over supports, $b_j\propto\hat{\theta}_j^q$, and use them directly for active search. A single placement can lie within 45\,cm of several adjacent supports, so these weights rank where to look rather than forming a distribution over mutually exclusive locations. For a category, the $\hat{\theta}_j^q$ remain independent occurrence probabilities.

Search actions are the executable views $u=(a,\psi)$ introduced in Section~\ref{sec:problem}. From the robot's current node $\mathrm{cur}$, executing view $u$ incurs cost $C_G(\mathrm{cur},u)+d_0$, where $C_G$ is shortest-path travel distance and $d_0>0$ accounts for acquiring an observation. A view may expose several supports; $c_j(u)$ denotes the opportunity to observe support $j$ from that view. We greedily choose the view that exposes the most current search weight per unit cost:
\begin{equation}
u^\star =
\arg\max_u
\frac{\sum_j b_j c_j(u)}
{C_G(\mathrm{cur},u)+d_0}.
\label{eq:index}
\end{equation}

After executing a view, the search terminates if the detector fires and verification confirms the target (in simulation, at least 500 target pixels are visible in the view; on the robot, an operator confirms it). Otherwise, the failed look provides negative evidence for the supports that were exposed. Let $p_{\mathrm{det}}$ be the probability that a visible target is successfully detected and verified, estimated from the calibration objects (Appendix~\ref{app:calibration}). We update
\begin{equation}
b'_j
\propto
b_j\bigl(1-p_{\mathrm{det}}c_j(u)\bigr).
\label{eq:update}
\end{equation}
Thus, a failed look strongly discounts well-observed supports, weakly discounts partially observed supports, and leaves unobserved supports unchanged. A detector trigger rejected by verification is treated as an unsuccessful look and therefore never raises the weight of the supports it exposed.

If the weights form a distribution over exclusive locations and view costs are fixed and order independent, ordering views by exposed probability per unit cost minimizes expected discovery cost (Appendix~\ref{app:planner}; \citealp{kadane1971optimal,stone1975theory}). We therefore use this criterion as a greedy heuristic, recomputing it after every observation to account for shared placements across adjacent supports and the robot's changing, position-dependent travel costs.

\section{LOOP-Bench: Long-Term Object Occurrence Prediction}
\label{sec:benchmark}

We introduce \textbf{LOOP-Bench} (\textbf{L}ong-Term \textbf{O}bject \textbf{O}ccurrence \textbf{P}rediction), a controlled benchmark that tests whether a robot can infer \emph{where an object usually occurs} from repeated, uneven observations. It contains ten HM3D-Semantics homes~\citep{ramakrishnan2021hm3d,yadav2023hm3dsem}, where objects are repeatedly placed from hidden long-term distributions while routes and viewpoints are sampled independently, separating prevalence from observation opportunity (Figure~\ref{fig:benchmark}).

Each home contains 20 past mapping traversals, giving 200 across the benchmark; each traversal provides RGB-D frames with a per-frame odometry pose in one common frame per home, from which a method may build any map. Each home has eight objects: five are evaluation queries, giving 50 queries in total, and three are calibration objects that are used only to select parameters and are excluded from evaluation. For each query, we evaluate both recovery of its ground-truth spatial occurrence distribution and closed-loop active search. Active search uses three independently sampled test placements per query, resulting in 150 search trials; because some placement distributions include an absent outcome, the object is not in the home in 22 of them. Full benchmark construction and evaluation details are provided in Appendix~\ref{app:benchmark}.

\begin{figure}[t]
\centering
\includegraphics[width=\linewidth]{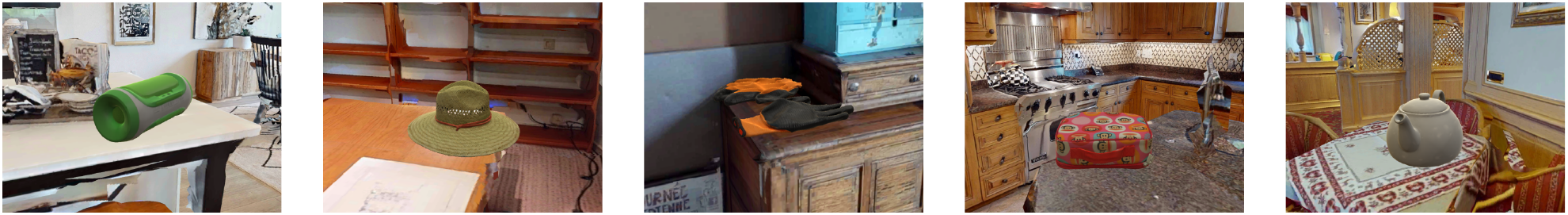}
\caption{\small Benchmark examples from five houses, rendered in the benchmark simulator. For every traversal each object is placed on a receptacle drawn from a hidden distribution and settled by physics; routes and gaze are sampled independently of the placements, so an object's presence cannot cause the camera to look at it.}
\label{fig:benchmark}
\end{figure}

\section{Simulation Experiments}
\label{sec:sim}

\begin{table}[t]
\caption{Retrospective distribution estimation and active search on LOOP-Bench (50 queries in 10 houses, 20 mapping traversals per house). AP, tie-aware Top-5 mass and AUROC rank the supports against the support-level ground truth; Bernoulli NLL scores each support's unnormalized prevalence against the probability that the object rests within 45\,cm of it (--: the score is not a probability). Search uses three test placements per query (150 trials) with a 120\,m budget; SR and SPL average over all 150 trials, and the 22 trials without the object count as failures. Bold marks the best value in each column across all rows; the lower block ablates \ECROM{}. Paired 95\% intervals are in Appendix~\ref{app:uncertainty}.}
\label{tab:sim-results}
\centering
\setlength{\tabcolsep}{3.4pt}
\begin{tabular}{lcccccc}
\toprule
Method & AP$\uparrow$ & Top-5$\uparrow$ & AUROC$\uparrow$ & $\mathrm{NLL}_{\mathrm{Ber}}\downarrow$ & SR$\uparrow$ & SPL$\uparrow$ \\
\midrule
\textbf{\ECROM{} (ours)} & \textbf{.216} & \textbf{.368} & \textbf{.760} & \textbf{.0287} & \textbf{.587} & \textbf{.381} \\
Detection count & .148 & .318 & .659 & -- & .440 & .330 \\
Visit frequency & .116 & .261 & .619 & .0811 & .433 & .299 \\
Opportunity-normalized count & .171 & .310 & .655 & .0773 & .467 & .310 \\
PredictiveGraphs-style (Perpetua$^\star$) & .160 & .263 & .659 & .0302 & .533 & .339 \\
RAVEN-style retrieval & .150 & .340 & .600 & -- & .420 & .310 \\
DynaMem-style latest state & .101 & .197 & .629 & -- & .400 & .267 \\
Coverage (no memory) & .011 & .006 & .500 & .0357 & .247 & .074 \\
\midrule
\ECROM{} without exposure term ($c_{je}{=}1$) & .191 & .342 & .718 & .0310 & .520 & .333 \\
\ECROM{} with binary exposure ($\ind[c_{je}{>}0]$) & .195 & .354 & .733 & \textbf{.0287} & .540 & .360 \\
\ECROM{} with binary sensor ($\tau{=}.25$) & .195 & .345 & .746 & .0297 & .473 & .339 \\
\ECROM{} with binary sensor ($\tau{=}.15$) & .175 & .318 & .752 & .0309 & .540 & .303 \\
\bottomrule
\end{tabular}
\end{table}

\subsection{Setup}
For each of the 50 LOOP-Bench queries, we build each method's memory from the 20 mapping traversals of its house and evaluate its predictions on the shared surface supports. We derive two ground truths from the hidden placement distribution \(\pi\): a 15\,cm footprint allocates each placement to supports, giving the occurrence probabilities \(g_j\) used for ranking, and \(\theta_j^\ast=\sum_k\pi_km_{jk}\) is the probability that the object rests within 45\,cm of support \(j\) (\(m_{jk}=1\) when placement \(k\) does), the prevalence of Eq.~\ref{eq:target} (Appendix~\ref{app:support-gt}).

We report AP, Top-5 mass, AUROC and Bernoulli NLL, averaged over the 50 queries. AP measures whether supports where the object actually occurs (\(y_j=\mathbf{1}[g_j>0]\)) are ranked ahead of those where it never occurs. Top-5 mass measures how much of the true occurrence probability lies in the five highest-ranked supports. AUROC is the probability that a support with \(y_j=1\) is scored above one with \(y_j=0\). These three metrics depend only on the order of the scores and measure whether a memory ranks first the supports where the object physically rests. \(\mathrm{NLL}_{\mathrm{Ber}}=-\frac{1}{J}\sum_j[\theta_j^\ast\log\hat\theta_j+(1-\theta_j^\ast)\log(1-\hat\theta_j)]\), with \(J=|\supportset|\) supports, instead scores the unnormalized prevalence estimates \(\hat\theta_j\) (Eq.~\ref{eq:target}) for the methods whose output is a probability, and changes when all prevalences are rescaled. Tie handling and clipping are specified in Appendix~\ref{app:support-gt}.

For active search, each method's estimated distribution is the search prior on the 150 test trials of Section~\ref{sec:benchmark}, with the same placement, start and \(120\,\mathrm{m}\) travel budget for every method. We report success rate (SR) and success weighted by path length (SPL)~\citep{anderson2018evaluation} over all trials, counting the 22 absent-target trials as failures; Appendix~\ref{app:uncertainty} reports the 128 present-target trials separately.

\subsection{Baselines}
\emph{Detection count} scores each support by the number of past traversals
in which the query is detected there, $\sum_e D_{je}$ with $D_{je}=\ind[r_{je}\geq\tau]$ and $\tau=.25$. \emph{Visit frequency} instead
normalizes this count by the number of traversals that visit the
corresponding place, estimating the frequency of detection per visit.
\emph{Opportunity-normalized count} divides the detection count by the summed opportunity, \(\min(1,\sum_e D_{je}/\sum_e c_{je})\); before clipping, this ratio is unbiased for a perfect binary detector (Appendix~\ref{app:model-properties}).
We further compare with \emph{RAVEN-style retrieval}~\citep{hu2026raven},
\emph{DynaMem-style latest-state memory}~\citep{liu2025dynamem}, and a
\emph{PredictiveGraphs-style} implementation of Perpetua$^\star$~\citep{saavedraruiz2026predictivegraphs,saavedraruiz2025perpetua}.
Finally, \emph{Coverage} uses no object memory and searches from a uniform
spatial prior. For a fair comparison, all methods receive the same stored observations, surface supports and detector
scores, including the same localized detections, so the rows differ only in how each memory turns this shared evidence into a belief, not in mapping or detection quality; the RAVEN-, DynaMem- and PredictiveGraphs-style rows reproduce the memory statistic of each system on these observations (Appendix~\ref{app:baselines}). For search, every memory is normalized into search weights in the same way and executed with the same planner (Section~\ref{sec:planner}).

\subsection{Results}

Table~\ref{tab:sim-results} shows that \ECROM{} gives the best long-term occurrence estimate on all four distribution metrics, raising AP from .171 for the strongest baseline to .216. Detection count and visit frequency perform worse because they confound occurrence with observation frequency; visit frequency is the weaker of the two because the visit denominator carries little information about the object. Dividing detections by opportunity instead of visits improves every ranking metric and gives the best AP among the baselines, but without a prior this count assigns zero prevalence to every exposed support without a detection, so its Bernoulli NLL is worse than the constant prior of Coverage (.0773 against .0357). RAVEN-style retrieval and DynaMem-style memory emphasize strong historical observations or the latest state rather than long-term prevalence. PredictiveGraphs-style filtering is the strongest baseline in Bernoulli NLL and in search, but it models the persistence and emergence of already detected object--support relations rather than aggregating repeated evidence into long-term prevalence. Supports with weak or missed detections therefore receive little probability mass, and recurring locations are harder to recover: \ECROM{} ranks the true supports clearly better (AUROC .760 against .659, with paired intervals for AP, Top-5 and AUROC that exclude zero), while the two methods are close in Bernoulli NLL (Table~\ref{tab:paired}).

These gains carry over to active search. \ECROM{} reaches .587 SR and .381 SPL, ahead of PredictiveGraphs-style memory (.533 and .339), the strongest baseline, although the paired intervals against it include zero (Table~\ref{tab:paired}); Coverage without memory succeeds in only a quarter of trials.

\subsection{Ablations}
\label{sec:ablation}

The lower block of Table~\ref{tab:sim-results} isolates the two main components of \ECROM{}. To test whether explicitly modeling observation opportunity is necessary, \ECROM{} without exposure term sets $c_{je}=1$ for every support in every traversal while keeping the same semantic evidence, likelihood ratio and posterior inference. Every distribution and search metric degrades; AP drops from .216 to .191 and SR from .587 to .520. A non-detection on a partially observed or unobserved support is now treated as if the support had been fully inspected, producing overly strong negative evidence: the prevalences of this variant sum to 2.32 per query, less than half of the true 4.76, whereas those of \ECROM{} sum to 5.10. The gain from the exposure term depends on how well the true supports were observed. When the queries are split into thirds by the opportunity summed over their true supports, the AP gain is $+.057$ for the least observed third and vanishes for the best observed third (Table~\ref{tab:mechanism}). Averaged over all queries the gain is therefore diluted, and its paired intervals include zero for every metric except Bernoulli NLL (Table~\ref{tab:paired}).

To separate missing observations from partial ones, \ECROM{} with binary exposure uses $\ind[c_{je}>0]$: it still ignores traversals that never observed the support, but treats every partial observation as a full inspection. It recovers part of the gap and matches \ECROM{}'s Bernoulli NLL, but ranks supports worse (AUROC .733 against .760) and searches less successfully, so fractional opportunity mainly sharpens the ranking of supports.

To test whether the graded likelihood ratio is useful, \ECROM{} with binary sensor replaces $\Lambda(r_{je})$ in Eq.~\ref{eq:posterior} by $\Lambda_{\mathrm{bin}}(D_{je})$ for a threshold $\tau\in\{.25,.15\}$, where $\Lambda_{\mathrm{bin}}(1)=h_\tau/f_\tau$ and $\Lambda_{\mathrm{bin}}(0)=(1-h_\tau)/(1-f_\tau)$, with hit rate $h_\tau=\Prob(D_{je}{=}1\mid Z_{je}{=}1)$ and false-alarm rate $f_\tau=\Prob(D_{je}{=}1\mid Z_{je}{=}0)$ computed from the same calibration histograms that define $\Lambda$ (Appendix~\ref{app:calibration}), so graded versus thresholded evidence is the only change. Both thresholds are worse than graded evidence on every metric, in different ways: the high threshold ($\tau=.25$) misses about half of the visible instances and has the lowest success rate of the variants (.473), while the low threshold ($\tau=.15$) admits four times as many false alarms and has their lowest AP (.175) and SPL. Graded evidence keeps weak responses without treating them like strong ones. Together, these ablations show that \ECROM{} benefits from modeling both \emph{how well a support was observed} and \emph{how strong the semantic evidence was}.

\begin{figure}[t]
\centering
\begin{minipage}[t]{0.57\linewidth}
\vspace{0pt}
\includegraphics[height=4.1cm]{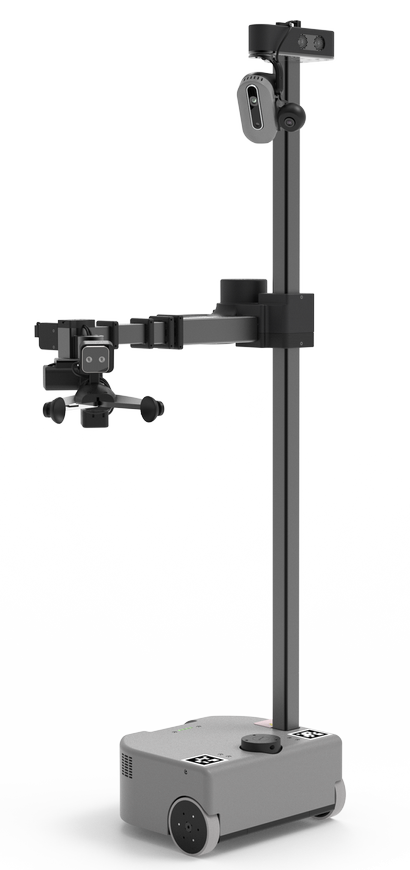}\hfill
\includegraphics[height=4.1cm]{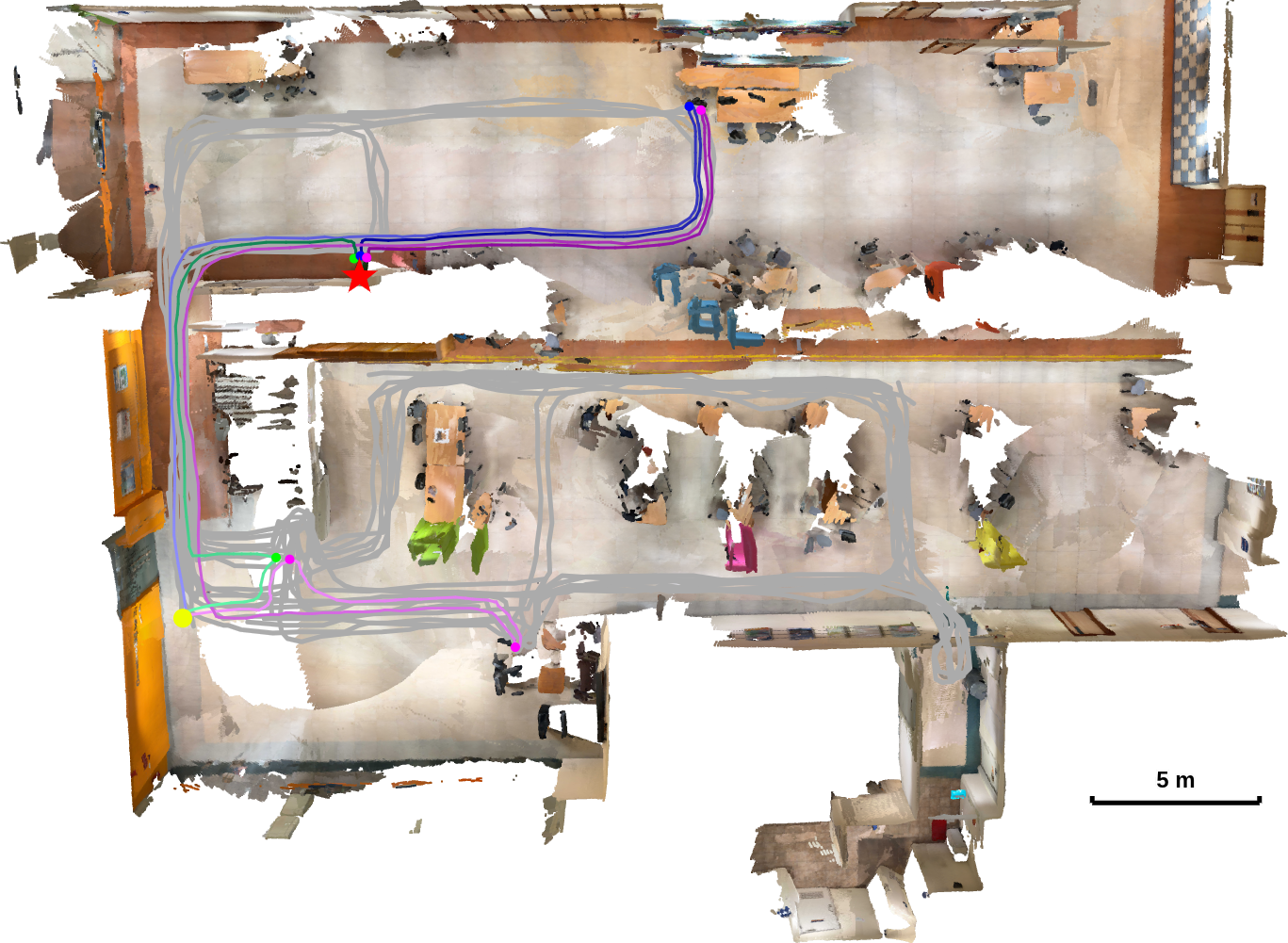}
\end{minipage}\hfill
\begin{minipage}[t]{0.40\linewidth}
\vspace{0pt}
\setlength{\abovecaptionskip}{0pt}
\caption{\small Real-robot experiment. \emph{Left}: Hello Robot Stretch~3 (image: Hello Robot). \emph{Right}: the $\approx$500\,m$^2$ office floor with the ten mapping traversals (grey) and, for one query, the search routes of \ECROM{} (green), detection count (blue) and PredictiveGraphs-style (magenta) from the start (yellow dot) to the target (red star). Scale bar 5\,m.}
\label{fig:realrobot}
\end{minipage}
\end{figure}

\section{Real-Robot Experiment}
\label{sec:realrobot}

\paragraph{Setup.}
We test whether the prevalence learned by \ECROM{} transfers to physical robot search. A Hello Robot Stretch~3 operates on an approximately $500\,\mathrm{m}^2$ office floor containing desks, tables, shelves, and other candidate supports (Figure~\ref{fig:realrobot}). Each object is assigned a hidden distribution over supports. Before each of ten mapping traversals, its location is resampled from this distribution; the robot records only RGB-D observations and odometry.

After mapping, we evaluate five queried objects with three independently sampled test placements each, giving 15 search trials. All methods share the start, detector, candidate views and planner. We compare \ECROM{} with detection count and PredictiveGraphs-style memory. A trial succeeds when the detector fires at an inspected view and an operator confirms the target within a budget of 20 inspected views. Table~\ref{tab:sim-vs-robot} lists every difference from the simulated system, including a belief update on the raw detector event $\ind[r\ge\tau]$. Failed trials count their travel until the budget ran out.

\begin{wraptable}{r}{0.44\linewidth}
\vspace{-12pt}
\setlength{\abovecaptionskip}{0pt}
\centering
\caption{\small Real-robot search on the Stretch~3 over 15 trials per method (five queried objects, three placements each; 20-view budget): success rate, mean travel and mean number of candidate views inspected. Per-query results are in Table~\ref{tab:realrobot-perquery}.}
\label{tab:realrobot}
\setlength{\tabcolsep}{3pt}
\resizebox{\linewidth}{!}{%
\begin{tabular}{lccc}
\toprule
Method & SR$\uparrow$ & Travel (m)$\downarrow$ & Views$\downarrow$ \\
\midrule
\textbf{\ECROM{} (ours)} & \textbf{.93} & \textbf{19.4} & \textbf{4.7} \\
Detection count & .67 & 29.1 & 9.3 \\
PredictiveGraphs-style & .73 & 29.5 & 8.9 \\
\bottomrule
\end{tabular}
}
\vspace{-14pt}
\end{wraptable}
\paragraph{Results.}
The real-robot results in Table~\ref{tab:realrobot} follow the same trend as simulation. \ECROM{} succeeds in 14 of 15 trials, against 10 for detection count and 11 for PredictiveGraphs-style memory, and reduces mean travel from about 29\,m to 19.4\,m and the number of inspected views from about nine to 4.7, so the prior focuses search on fewer, more promising locations.

Per query (Table~\ref{tab:realrobot-perquery}), the largest travel reductions occur for the red cup (25.4\,m against 42.2\,m for the better baseline) and the scissors (9.0\,m against 25.6\,m). These gains are consistent with uneven observation of their candidate supports. Detection count favors often-observed supports, whether or not the object is usually there, while PredictiveGraphs-style memory emphasizes the temporal state of established object--support relations; \ECROM{} instead keeps partial or missed observations from giving overly strong negative evidence.

For headphones and cupcakes, all three methods perform similarly, consistent with preferred locations that most mapping traversals observed. Luncheon meat is the hardest query: \ECROM{} finds it in two of three trials, compared with none for detection count and one for PredictiveGraphs-style memory; its one failure reaches the correct support, but the detector does not fire on the target. This pattern matches simulation, where the gain from opportunity is largest for the least observed queries (Section~\ref{sec:ablation}).

\section{Limitations}
\label{sec:limitations}

Our observation model makes two approximations. First, the exposure $c_{je}$ is a heuristic geometric proxy for the probability that a present object is observable: it equates observing half of a patch with a one-in-two chance of seeing the object, whereas this chance also depends on where on or near the patch the object rests, its size, the viewing distance and angle, and occlusion. Second, one likelihood ratio is shared by all queries, view counts and exposure levels, although weak responses carry more negative evidence when many views observed a support (Appendix~\ref{app:calibration}). A learned visibility model and a likelihood ratio conditioned on the observation process are directions for future work. Supports are geometric patches without receptacle semantics, which could add prior structure when observations are sparse. We also assume a stationary prevalence, whereas routines may change over longer deployments. Finally, the planner is greedy rather than optimizing the full route.

\section{Conclusion}

We introduced \ECROM{}, a retrospective open-vocabulary memory that estimates where objects usually occur from repeated, uneven observations. By separating semantic evidence from observation opportunity, \ECROM{} gives better long-term prevalence estimates and more effective search priors, in simulation and on a real robot, and the ablations show that both opportunity weighting and graded detector evidence contribute.

\bibliography{refs}

\begin{thebibliography}{29}
\providecommand{\natexlab}[1]{#1}
\providecommand{\url}[1]{\texttt{#1}}
\expandafter\ifx\csname urlstyle\endcsname\relax
  \providecommand{\doi}[1]{doi: #1}\else
  \providecommand{\doi}{doi: \begingroup \urlstyle{rm}\Url}\fi

\bibitem[Anderson et~al.(2018)Anderson, Chang, Chaplot, Dosovitskiy, Gupta,
  Koltun, Kosecka, Malik, Mottaghi, Savva, and Zamir]{anderson2018evaluation}
Peter Anderson, Angel Chang, Devendra~Singh Chaplot, Alexey Dosovitskiy,
  Saurabh Gupta, Vladlen Koltun, Jana Kosecka, Jitendra Malik, Roozbeh
  Mottaghi, Manolis Savva, and Amir~R. Zamir.
\newblock On evaluation of embodied navigation agents.
\newblock \emph{arXiv preprint arXiv:1807.06757}, 2018.

\bibitem[Batra et~al.(2020)Batra, Gokaslan, Kembhavi, Maksymets, Mottaghi,
  Savva, Toshev, and Wijmans]{batra2020objectnav}
Dhruv Batra, Aaron Gokaslan, Aniruddha Kembhavi, Oleksandr Maksymets, Roozbeh
  Mottaghi, Manolis Savva, Alexander Toshev, and Erik Wijmans.
\newblock {ObjectNav} revisited: On evaluation of embodied agents navigating to
  objects.
\newblock \emph{arXiv preprint arXiv:2006.13171}, 2020.

\bibitem[Chaplot et~al.(2020)Chaplot, Gandhi, Gupta, and
  Salakhutdinov]{chaplot2020semexp}
Devendra~Singh Chaplot, Dhiraj Gandhi, Abhinav Gupta, and Ruslan Salakhutdinov.
\newblock Object goal navigation using goal-oriented semantic exploration.
\newblock In \emph{Advances in Neural Information Processing Systems
  (NeurIPS)}, volume~33, pages 4247--4258, 2020.

\bibitem[Chen et~al.(2026)Chen, Kumar, Xu, Pavlakos, Biswas, and
  Mart{\'i}n-Mart{\'i}n]{chen2026star}
Taijing Chen, Sateesh Kumar, Junhong Xu, Georgios Pavlakos, Joydeep Biswas, and
  Roberto Mart{\'i}n-Mart{\'i}n.
\newblock Searching in space and time: Unified memory-action loops for
  open-world object retrieval.
\newblock In \emph{IEEE International Conference on Robotics and Automation
  (ICRA)}, 2026.

\bibitem[Deitke et~al.(2022)Deitke, VanderBilt, Herrasti, Weihs, Ehsani,
  Salvador, Han, Kolve, Kembhavi, and Mottaghi]{deitke2022procthor}
Matt Deitke, Eli VanderBilt, Alvaro Herrasti, Luca Weihs, Kiana Ehsani, Jordi
  Salvador, Winson Han, Eric Kolve, Aniruddha Kembhavi, and Roozbeh Mottaghi.
\newblock {ProcTHOR}: Large-scale embodied {AI} using procedural generation.
\newblock In \emph{Advances in Neural Information Processing Systems
  (NeurIPS)}, volume~35, pages 5982--5994, 2022.
\newblock \doi{10.52202/068431-0433}.

\bibitem[Downs et~al.(2022)Downs, Francis, Koenig, Kinman, Hickman, Reymann,
  McHugh, and Vanhoucke]{downs2022gso}
Laura Downs, Anthony Francis, Nate Koenig, Brandon Kinman, Ryan Hickman, Krista
  Reymann, Thomas~B. McHugh, and Vincent Vanhoucke.
\newblock {Google Scanned Objects}: A high-quality dataset of {3D} scanned
  household items.
\newblock In \emph{IEEE International Conference on Robotics and Automation
  (ICRA)}, pages 2553--2560, 2022.
\newblock \doi{10.1109/ICRA46639.2022.9811809}.

\bibitem[Gorlo et~al.(2026)Gorlo, Schmid, and Carlone]{gorlo2026daaam}
Nicolas Gorlo, Lukas Schmid, and Luca Carlone.
\newblock Describe anything anywhere at any moment.
\newblock In \emph{IEEE/CVF Conference on Computer Vision and Pattern
  Recognition (CVPR)}, pages 35002--35013, 2026.

\bibitem[Gu et~al.(2024)Gu, Kuwajerwala, Morin, Jatavallabhula, Sen, Agarwal,
  Rivera, Paul, Ellis, Chellappa, Gan, de~Melo, Tenenbaum, Torralba, Shkurti,
  and Paull]{gu2024conceptgraphs}
Qiao Gu, Alihusein Kuwajerwala, Sacha Morin, Krishna~Murthy Jatavallabhula,
  Bipasha Sen, Aditya Agarwal, Corban Rivera, William Paul, Kirsty Ellis, Rama
  Chellappa, Chuang Gan, Celso~Miguel de~Melo, Joshua~B. Tenenbaum, Antonio
  Torralba, Florian Shkurti, and Liam Paull.
\newblock {ConceptGraphs}: Open-vocabulary {3D} scene graphs for perception and
  planning.
\newblock In \emph{IEEE International Conference on Robotics and Automation
  (ICRA)}, pages 5021--5028, 2024.
\newblock \doi{10.1109/ICRA57147.2024.10610243}.

\bibitem[Hu et~al.(2026)Hu, Zheng, Zha, Xing, Singh, Hossain, Loquercio, and
  Shah]{hu2026raven}
Yixun Hu, Zhicheng Zheng, Lihan Zha, Chunwei Xing, Rajdeep Singh, Omar Hossain,
  Antonio Loquercio, and Dhruv Shah.
\newblock {RAVEN}: Long-horizon reasoning \& navigation with a
  visuo-spatio-temporal memory.
\newblock \emph{arXiv preprint arXiv:2606.25206}, 2026.

\bibitem[Huang et~al.(2023)Huang, Mees, Zeng, and Burgard]{huang2023vlmaps}
Chenguang Huang, Oier Mees, Andy Zeng, and Wolfram Burgard.
\newblock Visual language maps for robot navigation.
\newblock In \emph{IEEE International Conference on Robotics and Automation
  (ICRA)}, pages 10608--10615, 2023.
\newblock \doi{10.1109/ICRA48891.2023.10160969}.

\bibitem[Kadane(1971)]{kadane1971optimal}
Joseph~B. Kadane.
\newblock Optimal whereabouts search.
\newblock \emph{Operations Research}, 19\penalty0 (4):\penalty0 894--904, 1971.
\newblock \doi{10.1287/opre.19.4.894}.

\bibitem[Khanna et~al.(2024)Khanna, Ramrakhya, Chhablani, Yenamandra, Gervet,
  Chang, Kira, Chaplot, Batra, and Mottaghi]{khanna2024goat}
Mukul Khanna, Ram Ramrakhya, Gunjan Chhablani, Sriram Yenamandra, Theophile
  Gervet, Matthew Chang, Zsolt Kira, Devendra~Singh Chaplot, Dhruv Batra, and
  Roozbeh Mottaghi.
\newblock {GOAT-Bench}: A benchmark for multi-modal lifelong navigation.
\newblock In \emph{IEEE/CVF Conference on Computer Vision and Pattern
  Recognition (CVPR)}, pages 16373--16383, 2024.
\newblock \doi{10.1109/CVPR52733.2024.01549}.

\bibitem[Krajn{\'\i}k et~al.(2017)Krajn{\'\i}k, Fentanes, Santos, and
  Duckett]{krajnik2017fremen}
Tom{\'a}{\v{s}} Krajn{\'\i}k, Jaime~P. Fentanes, Jo{\~a}o~M. Santos, and Tom
  Duckett.
\newblock {FreMEn}: Frequency map enhancement for long-term mobile robot
  autonomy in changing environments.
\newblock \emph{IEEE Transactions on Robotics}, 33\penalty0 (4):\penalty0
  964--977, 2017.
\newblock \doi{10.1109/TRO.2017.2665664}.

\bibitem[Lee et~al.(2026)Lee, Jang, Heo, Choi, and Kim]{lee2026obsgraph}
Taekbeom Lee, Youngseok Jang, Jeonghwa Heo, Jeongjun Choi, and H.~Jin Kim.
\newblock {ObsGraph}: Hierarchical observation representation for embodied
  reasoning and exploration.
\newblock \emph{arXiv preprint arXiv:2606.24068}, 2026.

\bibitem[Liu et~al.(2025)Liu, Guo, Warke, Chintala, Paxton, Shafiullah, and
  Pinto]{liu2025dynamem}
Peiqi Liu, Zhanqiu Guo, Mohit Warke, Soumith Chintala, Chris Paxton, Nur
  Muhammad~Mahi Shafiullah, and Lerrel Pinto.
\newblock {DynaMem}: Online dynamic spatio-semantic memory for open world
  mobile manipulation.
\newblock In \emph{IEEE International Conference on Robotics and Automation
  (ICRA)}, pages 13346--13355, 2025.
\newblock \doi{10.1109/ICRA55743.2025.11127619}.

\bibitem[Minderer et~al.(2023)Minderer, Gritsenko, and
  Houlsby]{minderer2023owlv2}
Matthias Minderer, Alexey Gritsenko, and Neil Houlsby.
\newblock Scaling open-vocabulary object detection.
\newblock In \emph{Advances in Neural Information Processing Systems
  (NeurIPS)}, volume~36, pages 72983--73007, 2023.
\newblock \doi{10.52202/075280-3191}.

\bibitem[Ramakrishnan et~al.(2021)Ramakrishnan, Gokaslan, Wijmans, Maksymets,
  Clegg, Turner, Undersander, Galuba, Westbury, Chang, Savva, Zhao, and
  Batra]{ramakrishnan2021hm3d}
Santhosh~K. Ramakrishnan, Aaron Gokaslan, Erik Wijmans, Oleksandr Maksymets,
  Alexander Clegg, John Turner, Eric Undersander, Wojciech Galuba, Andrew
  Westbury, Angel~X. Chang, Manolis Savva, Yili Zhao, and Dhruv Batra.
\newblock {Habitat-Matterport} {3D} dataset ({HM3D}): 1000 large-scale {3D}
  environments for embodied {AI}.
\newblock In \emph{Proceedings of the Neural Information Processing Systems
  Track on Datasets and Benchmarks}, volume~1, 2021.

\bibitem[Ramakrishnan et~al.(2022)Ramakrishnan, Chaplot, Al-Halah, Malik, and
  Grauman]{ramakrishnan2022poni}
Santhosh~K. Ramakrishnan, Devendra~Singh Chaplot, Ziad Al-Halah, Jitendra
  Malik, and Kristen Grauman.
\newblock {PONI}: Potential functions for {ObjectGoal} navigation with
  interaction-free learning.
\newblock In \emph{IEEE/CVF Conference on Computer Vision and Pattern
  Recognition (CVPR)}, pages 18868--18878, 2022.
\newblock \doi{10.1109/CVPR52688.2022.01832}.

\bibitem[Saavedra-Ruiz et~al.(2025)Saavedra-Ruiz, Nashed, Gauthier, and
  Paull]{saavedraruiz2025perpetua}
Miguel Saavedra-Ruiz, Samer~B. Nashed, Charlie Gauthier, and Liam Paull.
\newblock Perpetua: Multi-hypothesis persistence modeling for semi-static
  environments.
\newblock In \emph{IEEE/RSJ International Conference on Intelligent Robots and
  Systems (IROS)}, pages 13843--13850, 2025.
\newblock \doi{10.1109/IROS60139.2025.11247086}.

\bibitem[Saavedra-Ruiz et~al.(2026)Saavedra-Ruiz, Gauthier, Gupta, Shahfar,
  Ellis, Parkison, and Paull]{saavedraruiz2026predictivegraphs}
Miguel Saavedra-Ruiz, Charlie Gauthier, Kumaraditya Gupta, Shima Shahfar,
  Kirsty Ellis, Steven Parkison, and Liam Paull.
\newblock Predictive spatio-temporal scene graphs for semi-static scenes.
\newblock \emph{arXiv preprint arXiv:2605.00121}, 2026.

\bibitem[Savva et~al.(2019)Savva, Kadian, Maksymets, Zhao, Wijmans, Jain,
  Straub, Liu, Koltun, Malik, Parikh, and Batra]{savva2019habitat}
Manolis Savva, Abhishek Kadian, Oleksandr Maksymets, Yili Zhao, Erik Wijmans,
  Bhavana Jain, Julian Straub, Jia Liu, Vladlen Koltun, Jitendra Malik, Devi
  Parikh, and Dhruv Batra.
\newblock Habitat: A platform for embodied {AI} research.
\newblock In \emph{IEEE/CVF International Conference on Computer Vision
  (ICCV)}, pages 9338--9346, 2019.
\newblock \doi{10.1109/ICCV.2019.00943}.

\bibitem[Stone(1975)]{stone1975theory}
Lawrence~D. Stone.
\newblock \emph{Theory of Optimal Search}, volume 118 of \emph{Mathematics in
  Science and Engineering}.
\newblock Academic Press, New York, 1975.
\newblock \doi{10.1016/S0076-5392(08)X6033-3}.

\bibitem[Wald et~al.(2019)Wald, Avetisyan, Navab, Tombari, and
  Nie{\ss}ner]{wald2019rio}
Johanna Wald, Armen Avetisyan, Nassir Navab, Federico Tombari, and Matthias
  Nie{\ss}ner.
\newblock {RIO}: {3D} object instance re-localization in changing indoor
  environments.
\newblock In \emph{IEEE/CVF International Conference on Computer Vision
  (ICCV)}, pages 7657--7666, 2019.
\newblock \doi{10.1109/ICCV.2019.00775}.

\bibitem[Wang and Soh(2023)]{wang2023smallobj}
Jiaming Wang and Harold Soh.
\newblock Small object navigation with context information.
\newblock \emph{Proceedings of the AAAI Symposium Series}, 1\penalty0
  (1):\penalty0 117--120, 2023.
\newblock \doi{10.1609/aaaiss.v1i1.27487}.

\bibitem[Wang and Soh(2024)]{wang2024polo}
Jiaming Wang and Harold Soh.
\newblock Probable object location ({POLo}) score estimation for efficient
  object goal navigation.
\newblock In \emph{IEEE International Conference on Robotics and Automation
  (ICRA)}, pages 5221--5227, 2024.
\newblock \doi{10.1109/ICRA57147.2024.10610671}.

\bibitem[Wang et~al.(2026)Wang, Chen, Liu, Ghotavadekar, Da, K{\"a}stner, and
  Soh]{wang2026cross}
Jiaming Wang, Jizhuo Chen, Diwen Liu, Atharva Ghotavadekar, Jiaxuan Da, Linh
  K{\"a}stner, and Harold Soh.
\newblock Change-robust online spatial-semantic topological mapping.
\newblock \emph{arXiv preprint arXiv:2605.02227}, 2026.

\bibitem[Yadav et~al.(2023)Yadav, Ramrakhya, Ramakrishnan, Gervet, Turner,
  Gokaslan, Maestre, Chang, Batra, Savva, Clegg, and Chaplot]{yadav2023hm3dsem}
Karmesh Yadav, Ram Ramrakhya, Santhosh~Kumar Ramakrishnan, Theo Gervet, John
  Turner, Aaron Gokaslan, Noah Maestre, Angel~Xuan Chang, Dhruv Batra, Manolis
  Savva, Alexander~William Clegg, and Devendra~Singh Chaplot.
\newblock {Habitat-Matterport} {3D} semantics dataset.
\newblock In \emph{IEEE/CVF Conference on Computer Vision and Pattern
  Recognition (CVPR)}, pages 4927--4936, 2023.
\newblock \doi{10.1109/CVPR52729.2023.00477}.

\bibitem[Yan et~al.(2025)Yan, Li, Wang, Wu, Wang, Zhu, Chen, and
  Liu]{yan2025dovsg}
Zhijie Yan, Shufei Li, Zuoxu Wang, Lixiu Wu, Han Wang, Jun Zhu, Lijiang Chen,
  and Jihong Liu.
\newblock Dynamic open-vocabulary {3D} scene graphs for long-term
  language-guided mobile manipulation.
\newblock \emph{IEEE Robotics and Automation Letters}, 10\penalty0
  (5):\penalty0 4252--4259, 2025.
\newblock \doi{10.1109/LRA.2025.3547643}.

\bibitem[Zhai et~al.(2023)Zhai, Mustafa, Kolesnikov, and Beyer]{zhai2023siglip}
Xiaohua Zhai, Basil Mustafa, Alexander Kolesnikov, and Lucas Beyer.
\newblock Sigmoid loss for language image pre-training.
\newblock In \emph{IEEE/CVF International Conference on Computer Vision
  (ICCV)}, pages 11941--11952, 2023.
\newblock \doi{10.1109/ICCV51070.2023.01100}.

\end{thebibliography}
\bibliographystyle{plainnat}

\clearpage
\appendix

\section{Derivations for the Censored Observation Model}
\label{app:derivation}

This appendix derives the opportunity-weighted likelihood of
Proposition~\ref{prop:likelihood}, the prevalence posterior of
Proposition~\ref{prop:posterior}, and several useful consequences of the
model. For readability, we first suppress the support, episode, and query
indices.

\subsection{Derivation of the opportunity-weighted likelihood}
\label{app:likelihood-derivation}

Recall the latent model
\begin{equation}
X\mid\theta\sim\Ber(\theta),
\qquad
V\mid c\sim\Ber(c),
\qquad
Z=XV,
\end{equation}
where $X$ denotes physical presence, $V$ whether the support is observable,
and $Z$ whether the queried object is both present and observable. We assume
$X\perp V$ conditioned on $\theta$ and $c$. Therefore
\begin{equation}
\Prob(Z=1\mid\theta,c)
=
\Prob(X=1\mid\theta)\Prob(V=1\mid c)
=
\theta c,
\end{equation}
and hence
\begin{equation}
\Prob(Z=0\mid\theta,c)=1-\theta c.
\end{equation}

The detector score $r$ follows the class-conditional distributions
\begin{equation}
p(r\mid Z=z)=p_z(r),
\qquad z\in\{0,1\}.
\end{equation}
Marginalizing the latent visible-presence state gives
\begin{align}
p(r\mid\theta,c)
&=
\sum_{z\in\{0,1\}}
p(r\mid Z=z)\Prob(Z=z\mid\theta,c) \\
&=
(1-\theta c)p_0(r)+\theta c\,p_1(r).
\end{align}
Using the likelihood ratio
\begin{equation}
\Lambda(r)=\frac{p_1(r)}{p_0(r)},
\end{equation}
we obtain
\begin{equation}
p(r\mid\theta,c)
=
p_0(r)
\left[
1+\theta c\bigl(\Lambda(r)-1\bigr)
\right],
\end{equation}
which is Eq.~\ref{eq:likelihood} and proves
Proposition~\ref{prop:likelihood}.

This form makes the roles of the two measurements explicit:
$\Lambda(r)$ determines whether the observation favors presence or absence,
while $c$ determines how strongly that evidence should affect prevalence.

\subsection{Derivation of the prevalence posterior}
\label{app:posterior-derivation}

Returning to the full indices, let
\begin{equation}
\mathcal H_j=\{(r_{je},c_{je})\}_{e=1}^{E}
\end{equation}
be the history associated with support $j$, and place the prior
\begin{equation}
\theta_j\sim\BetaDist(\alpha_0,\beta_0).
\end{equation}
Assuming episodes are conditionally independent given $\theta_j$ and their
recorded opportunities, and writing $\Lambda_{je}=\Lambda(r_{je})$,
\begin{align}
p(\theta_j\mid\mathcal H_j,q)
&\propto
p(\theta_j)
\prod_{e=1}^{E}
p(r_{je}\mid\theta_j,c_{je}) \\
&\propto
\theta_j^{\alpha_0-1}
(1-\theta_j)^{\beta_0-1}
\prod_{e=1}^{E}
p_0(r_{je})
\left[
1+\theta_jc_{je}(\Lambda_{je}-1)
\right].
\end{align}
The factors $p_0(r_{je})$ do not depend on $\theta_j$ and are absorbed into
the normalizing constant, leaving
\begin{equation}
p(\theta_j\mid\mathcal H_j,q)
\propto
\theta_j^{\alpha_0-1}
(1-\theta_j)^{\beta_0-1}
\prod_{e=1}^{E}
\left[
1+\theta_jc_{je}(\Lambda_{je}-1)
\right],
\end{equation}
which is Eq.~\ref{eq:posterior} and proves
Proposition~\ref{prop:posterior}.

\subsection{Episode-level presence}
\label{app:episode-presence}

The prevalence posterior answers where the object \emph{usually} occurs.
The same model can also estimate whether it was present in a particular
episode.

Conditioned on $X=1$, the support is observable with probability $c$, giving
\begin{align}
p(r\mid X=1,c)
&=
(1-c)p_0(r)+c\,p_1(r) \\
&=
p_0(r)\left[(1-c)+c\Lambda(r)\right].
\end{align}
Bayes' rule therefore gives
\begin{equation}
\Prob(X=1\mid r,\theta,c)
=
\frac{
\theta\left[(1-c)+c\Lambda(r)\right]
}{
1+\theta c(\Lambda(r)-1)
}.
\label{eq:gamma-theta}
\end{equation}
Integrating over the inferred prevalence yields
\begin{equation}
\Prob(X_{je}=1\mid\mathcal H_j,q)
=
\int_0^1
\Prob(X_{je}=1\mid r_{je},\theta_j,c_{je})
\,p(\theta_j\mid\mathcal H_j,q)\,
d\theta_j.
\end{equation}
Thus the model supports both retrospective episode queries
(``where was it then?'') and long-term prevalence queries
(``where is it usually?'') without changing the stored representation.

\subsection{Special cases and invariances}
\label{app:model-properties}

\paragraph{No observation opportunity.}
If $c_{je}=0$, then
\begin{equation}
1+\theta_jc_{je}(\Lambda_{je}-1)=1
\end{equation}
for every $\theta_j$. The episode therefore contributes no information,
regardless of its detector score. This is the censoring property: an
unobserved support is not treated as an observed absence.

\paragraph{Positive, negative, and neutral semantic evidence.}
For $c_{je}>0$, the likelihood factor
\begin{equation}
1+\theta_jc_{je}(\Lambda_{je}-1)
\end{equation}
increases with $\theta_j$ when $\Lambda_{je}>1$ and decreases with
$\theta_j$ when $\Lambda_{je}<1$. When $\Lambda_{je}=1$, it equals one and
the detector observation is uninformative. In all cases, its departure from
one vanishes continuously as $c_{je}\rightarrow0$.

\paragraph{Full opportunity and a perfect detector.}
Suppose $c_{je}=1$ and the observation is a perfect binary detection
$D_{je}=X_{je}$. Then
\begin{equation}
p(D_{je}\mid\theta_j)
=
\theta_j^{D_{je}}
(1-\theta_j)^{1-D_{je}},
\end{equation}
and the posterior reduces to the standard Beta--Bernoulli update
\begin{equation}
\theta_j\mid\{D_{je}\}_{e=1}^{E}
\sim
\BetaDist
\left(
\alpha_0+\sum_e D_{je},
\;
\beta_0+E-\sum_e D_{je}
\right).
\end{equation}
Thus the proposed model generalizes ordinary prevalence estimation to
partial observation and uncertain semantic evidence.

\paragraph{Perfect detector with varying opportunity.}
If a perfect binary sensor reports visible presence,
$D_{je}=Z_{je}$, then
\begin{equation}
D_{je}\sim\Ber(\theta_jc_{je}).
\end{equation}
Consequently,
\begin{equation}
\widehat{\theta}_{j,\mathrm{opp}}
=
\frac{\sum_e D_{je}}{\sum_e c_{je}}
\end{equation}
is conditionally unbiased whenever $\sum_e c_{je}>0$, because
\begin{equation}
\E\left[\sum_eD_{je}\,\middle|\,\{c_{je}\}_e\right]
=
\theta_j\sum_ec_{je};
\end{equation}
at finite sample size the unconstrained ratio can exceed one, and clipping it to $[0,1]$ introduces bias, which is one reason to prefer the posterior of Eq.~\ref{eq:posterior} over the ratio. With a detector thresholded at $\tau=.25$ in place of the perfect sensor, this ratio is the opportunity-normalized count of Table~\ref{tab:sim-results}.
In contrast, normalizing by the number of visits estimates detections per
visit and is biased whenever visiting a place does not imply observing the
support.

\paragraph{Frame duplication.}
Our episode-level quantities are invariant to duplicated frames:
$c_{je}$ is computed from the union of observed surface voxels and
$r_{je}$ is the maximum localized score within the episode. Repeating an
identical frame therefore changes neither quantity and hence leaves the
posterior unchanged. Distinct views do change the distribution of the maximum; Table~\ref{tab:lr-audit} measures this on the calibration objects.

\section{Motivation for the Probability-per-Cost Index}
\label{app:planner}

The index rule of Eq.~\ref{eq:index} is motivated by a simplified search problem in which it is exactly optimal. We state that problem, give the classical interchange argument~\citep{kadane1971optimal,stone1975theory}, and then discuss which of its assumptions our navigation setting violates. In our setting the rule is therefore a greedy heuristic rather than an optimal policy; we use it for its simplicity and leave a globally optimal search planner to future work.

\paragraph{Simplified setting.} Views $u\in\mathcal U$ are executed at most once; view $u$ has a fixed cost $C_u>0$ that does not depend on the order of execution, the weights $b_j$ form a distribution over mutually exclusive target locations, and a single inspection of $u$ discovers the target with probability $\rho_u=p_{\mathrm{det}}\sum_jb_jc_j(u)$ evaluated at the initial belief, independently across views. A false alarm does not end the search and therefore does not enter $\rho_u$. A schedule $\sigma=(u_1,u_2,\dots)$ has expected cost until detection $J(\sigma)=\sum_{t\ge1}C_{u_t}\prod_{s<t}(1-\rho_{u_s})$.

\begin{proposition}[Index rule]
Under the simplified setting, the schedule that sorts views by decreasing $\rho_u/C_u$ minimizes $J$, and since $\rho_u=p_{\mathrm{det}}\sum_jb_jc_j(u)$ with $p_{\mathrm{det}}$ common to all views, this is the order of Eq.~\ref{eq:index}.
\end{proposition}
\emph{Proof.} Take a schedule in which view $a$ immediately precedes view $b$ with $\rho_a/C_a<\rho_b/C_b$, and let $P$ be the probability of reaching this pair undetected. Swapping them changes $J$ only through the terms $P[C_a+(1-\rho_a)C_b]$ versus $P[C_b+(1-\rho_b)C_a]$; their difference is $P(\rho_bC_a-\rho_aC_b)=P\,C_aC_b\,(\rho_b/C_b-\rho_a/C_a)>0$, so the swap strictly lowers $J$. Repeated exchanges sort the schedule and no further exchange improves it. $\square$

\paragraph{What the navigation setting violates.} Four assumptions fail in Section~\ref{sec:planner}. First, a single placement can lie within 45\,cm of several adjacent supports, so the normalized prevalences $b_j$ are search weights rather than a distribution over exclusive locations, and $p_{\mathrm{det}}\sum_jb_jc_j(u)$ is not the probability that view $u$ finds the target. Second, the cost $C(\mathrm{cur},u)=C_G(\mathrm{cur},u)+d_0$ depends on the robot's current node, so view costs are not order independent. Third, a view exposes several supports at once, so detection probabilities of different views are coupled through the shared belief and the update of Eq.~\ref{eq:update} changes the ratios $\rho_u/C_u$ of the remaining views after every observation. Fourth, a detection is followed by verification, and a rejected trigger costs the verification without ending the search. For these reasons we use $\rho_u/C_u$ as a greedy heuristic, recomputed from the current weights and current position after every executed view, rather than as a precomputed schedule.

\section{Compared Memories: Formulations and Implementation}
\label{app:baselines}

\paragraph{Common inputs.}
Every compared memory is re-implemented on the same stored keyframes, the same surface-patch supports $j$, and the same OWLv2 scores, so the methods differ only in the statistic they compute from the history. Stored frame $n$ carries its query score $r_n$, its support $j(n)$ (the support whose surface its localized detection was lifted to; frames without a localized detection have none), its place, and its episode $e(n)$. The episode evidence $r_{je}=\max_{n:\,j(n)=j,\,e(n)=e}r_n$ is shared by all support-level statistics; detections are $D_{je}=\ind[r_{je}\ge\tau]$ with $\tau=.25$ for every method; $v_{je}=1$ when episode $e$ stored at least one frame at the place of support $j$; $c_{je}$ is the geometric opportunity of Eq.~\ref{eq:exposure}; and $E$ is the number of mapping episodes. Retrospective estimates use all 20 mapping traversals of a house. For the ranking metrics of Table~\ref{tab:sim-results} (AP, Top-5, AUROC), every score vector is compared directly with the footprint-based ground truth on the same supports (Section~\ref{sec:sim}). Bernoulli NLL scores the unnormalized estimates of the memories whose output is a probability: \ECROM{} and its variants, visit frequency, the opportunity-normalized count, the PredictiveGraphs-style filter, and Coverage, whose constant score is the prior mean; detection counts and the RAVEN- and DynaMem-style scores are not probabilities and are not scored. For search, count, frequency, latest-state, and persistence estimates are clipped, normalized into search weights, and planned with the index rule of Eq.~\ref{eq:index} and the update of Eq.~\ref{eq:update} exactly as the \ECROM{} weights; the RAVEN-style scores are normalized and planned the same way; supports outside the retrieved set carry no mass, so this policy stops once the retrieved supports are inspected. Coverage plans the same rule on a uniform belief, i.e., it maximizes mean opportunity per meter; every belief-based policy stops when no remaining candidate exposes belief mass. The relation filters' duration mixtures, detector rates and hyperparameters are fitted on the calibration objects; every other baseline has the fixed parameters listed below, and none sees the evaluation queries.

\paragraph{Detection count.}
The score is
\begin{equation}
s_j^{\rm count}=\frac{\sum_eD_{je}}{\max_{j'}\sum_eD_{j'e}},
\end{equation}
the number of traversals in which the query was detected at support $j$, divided by its maximum over supports. It has no denominator for opportunity: $\E[\sum_eD_{je}]$ scales with both $\theta_j$ and $\sum_ec_{je}$.

\paragraph{Visit frequency.}
\begin{equation}
s_j^{\rm freq}=\frac{\sum_eD_{je}v_{je}}{\sum_ev_{je}}.
\end{equation}
This adds a denominator, but it is the wrong one whenever the robot visits a place without viewing support $j$. It estimates detection per visit rather than occurrence per useful look. It is the \emph{visit frequency} baseline of Table~\ref{tab:sim-results}.

\paragraph{Opportunity-normalized count.}
\begin{equation}
s_j^{\rm opp}=\min\Bigl(1,\frac{\sum_eD_{je}}{\sum_ec_{je}}\Bigr)\quad\text{if }\textstyle\sum_ec_{je}>0,\qquad s_j^{\rm opp}=\frac{\alpha_0}{\alpha_0+\beta_0}\ \text{otherwise}.
\end{equation}
This replaces the visit denominator with the geometric opportunity of Eq.~\ref{eq:exposure}; it is the perfect-sensor estimator of Appendix~\ref{app:model-properties} applied to thresholded detections, without a prior or graded evidence.

\paragraph{RAVEN-style retrieval.}
The memory ranks stored frames by open-vocabulary score, retrieves the top $K=30$, attaches pose and episode, and lets frame $n$ vote for its support with
\begin{equation}
s_j^{\rm ret}=\max_{n:\,j(n)=j} r_n\exp[-(E-e(n))/\tau_r],
\end{equation}
with recency measured in episodes from the most recent mapped traversal and $\tau_r=12$ episodes; frames without a localized support are not retrievable, and supports outside the top $K$ score zero. This is neither a count nor a posterior. It is a recency-weighted maximum designed to find a memorable historical observation. It ignores exposed non-retrievals; consequently it cannot distinguish one response in one opportunity from one response in one hundred opportunities.

\paragraph{DynaMem-style latest state.}
For each support, find the latest episode $e^*(j)$ whose opportunity reaches a threshold, $c_{je}\ge.5$, and keep
\begin{equation}
s_j^{\rm latest}=r_{j,e^*(j)},
\end{equation}
with score zero for supports never re-observed at that level. This approximates the frustum-based overwriting of a dynamic spatial memory: a region that is re-observed replaces its earlier content. It is useful for current state, but two histories with the same final view and radically different long-term frequencies receive the same score.

\paragraph{PredictiveGraphs-style relation filters (Perpetua$^\star$).}
This is the closest probabilistic baseline, so we re-implement its estimator rather than a simplification of it. PredictiveGraphs attaches a Perpetua$^\star$ filter to every object--receptacle edge of a scene graph~\citep{saavedraruiz2026predictivegraphs,saavedraruiz2025perpetua}. An edge is instantiated for a support only if the query produced at least one \emph{localized} accepted positive there within the traversals the estimator may use, never from held-out traversals or placement labels; a support with no edge carries the base occurrence rate of never-detected supports, measured on the calibration objects.

Each edge carries two survival-time models over the binary relation state, both updated with every valid observation: \emph{persistence} $M_P$, under which the relation holds until a survival time $T$, and \emph{emergence} $M_E$, under which it is absent until an emergence time $T_E$. Each is a mixture over duration hypotheses $\kappa_l$ with categorical prior $\nu_l$, and with binary detector emissions $P_M=\Prob(y=0\mid X=1)$, $P_F=\Prob(y=1\mid X=0)$ the conditional evidence of a component has the closed form
\begin{equation}
\Prob(\mathcal Y\mid \kappa_l)=\sum_{k}\Big[\textstyle\prod_{i\le k}\Prob(y_i\mid \text{pre})\prod_{i>k}\Prob(y_i\mid \text{post})\Big]\,\Prob\!\left(T\in[t_k-t_0,\,t_{k+1}-t_0)\mid \kappa_l\right),
\end{equation}
summed over the edge's own observation times with $t_0$ its first valid observation, which yields the posterior component weights $w_l\propto\nu_l\Prob(\mathcal Y\mid \kappa_l)$; each model predicts from its dominant component. The two are combined by Bayesian model selection against a switching prior $p_t(M)$ given by a schedule $f(t)$ and an annealing factor $\zeta(t)=\exp(-\eta(t-t_N))$, with $t_N$ the edge's last valid observation,
\begin{equation}
\begin{aligned}
\Prob(M\mid\mathcal Y,t)&\propto\Prob(\mathcal Y\mid M)^{\zeta(t)}p_t(M),\\
\Prob(X_t=1\mid\mathcal Y)&=\sum_{M}\Prob(M\mid\mathcal Y,t)\,\Prob(X_t=1\mid\mathcal Y,\kappa_{l^*},M),
\end{aligned}
\end{equation}
the sum running over $M_P$ and $M_E$, and with a forgetting factor $\gamma$ discounting the running measurement likelihood. An observation exists for a traversal when the support was exposed, $c_{je}\ge c_{\min}$, or produced an accepted positive, which overrides the gate; other traversals are missing and advance the clock without an update.

The estimator is the published one: our implementation reproduces the authors' reference code to $1.5\times10^{-5}$ in log evidence and $6.8\times10^{-7}$ in the predicted probability. Table~\ref{tab:pg-impl} separates what is inherited from the original model, what our setting requires us to adapt, and which quantities are fitted on the calibration objects. Three adaptations are required by our setting. First, the location units are the shared $0.5$\,m supports, not semantically grouped receptacles. Second, the clock is the traversal index, so $\eta$ is per traversal and $\gamma$ per update. Third, the switching prior is order zero, $f(t)=f_0$, obtained from the edge's own observed history after correcting the detection fraction for $P_M$ and $P_F$ and shrinking it towards the pooled calibration rate: the FreMEn and LLM schedules of the original method need periodic structure over weeks of hourly observations, and our 20 traversals with independently resampled placements contain none. Because the filter forecasts forward in time, the retrospective estimate compared with the ground-truth distribution is its mean prediction over five horizons after the last mapping traversal, and the search prior is its prediction one traversal after the last mapping traversal.

\begin{table}[t]
\caption{PredictiveGraphs-style baseline: what is inherited from the original model, what our setting requires, and what is fitted. All fitting uses the calibration objects only.}
\label{tab:pg-impl}
\centering
\small
\setlength{\tabcolsep}{5pt}
\begin{tabularx}{\linewidth}{>{\raggedright\arraybackslash}p{2.9cm}>{\raggedright\arraybackslash}X>{\raggedright\arraybackslash}X}
\toprule
Component & Original PredictiveGraphs (Perpetua$^\star$) & Our implementation on LOOP-Bench \\
\midrule
Relation state & binary object--receptacle relation & binary object--support relation on the shared $0.5$\,m supports \\
Edge instantiation & first accepted detection & first localized accepted positive ($r_{je}\ge.25$) within the mapping traversals; supports without an edge carry the base rate of never-detected supports \\
Temporal models & persistence $M_P$ and emergence $M_E$, log-normal duration mixtures, EM-fitted & unchanged; one to three components selected by AIC, fitted by the reference EM on the calibration histories with gaps preserved \\
Detector emissions & $P_M$, $P_F$ & unchanged; pooled rates of the calibration objects \\
Model selection & Bayesian selection with annealing $\eta$ and forgetting $\gamma$ & unchanged; clock is the traversal index, so $\eta$ is per traversal and $\gamma$ per update \\
Switching prior $f(t)$ & FreMEn / LLM schedules over weeks of hourly data & order zero $f_0$ from the edge's own corrected detection fraction, shrunk towards the pooled rate (placements are independent across the 20 traversals) \\
Missing observations & advance the clock without an update & unchanged; an observation exists when $c_{je}\ge c_{\min}$ or an accepted positive occurred \\
Hyperparameters $c_{\min},\gamma,\eta$ & -- & selected on the calibration objects by held-out AP subject to NLL within $5\%$ of its minimum \\
Output & next-state prediction & mean prediction over five horizons after the last mapping traversal (distribution metrics); one-step prediction (search prior) \\
\bottomrule
\end{tabularx}
\end{table}

\paragraph{Selected values.} The hyperparameters selected on the calibration objects are $c_{\min}=.1$, $\gamma=.95$ and $\eta=.2$ per traversal, with $\tau=.25$ and the $3$\,m detection range shared with every method. The detector emissions are the pooled rates of the calibration objects, $P_M=.709$ and $P_F=.0019$, for every query. Both duration mixtures select two log-normal components by AIC on $433$ calibration edges: persistence weights $(.52,.48)$, log-means $(1.20,2.40)$ and log-standard deviations $(.23,.55)$; emergence weights $(.81,.19)$, log-means $(-.71,5.01)$ and log-standard deviations $(.53,1.36)$. The order-zero switching prior of an edge is its detection fraction corrected for $P_M$ and $P_F$ and shrunk with two pseudo-counts towards the pooled occupancy $.164$; supports without an edge carry the base rate $.0020$ of never-detected supports; the retrospective estimate averages the forecasts at horizons one to five.

The remaining gap is the estimand: its latent variable is the survival of an established relation, not the stationary prevalence of independently re-placed instances. One consequence is visible in Table~\ref{tab:sim-results}: its candidate set is restricted to supports where the query was actually detected, which hold $.51$ of the ground-truth occurrence mass on average, so its AP is bounded by candidate construction rather than by the filter, whereas \ECROM{} can assign mass to supports that produced only weak positives.

\paragraph{Summary of the gap.}
Retrieval and latest-state methods keep \emph{which positive was strongest or most recent}; counts estimate \emph{positive evidence per acquisition unit}; persistence filters estimate \emph{the next state of an established relation}. \ECROM{} estimates \emph{physical occurrence per chance to observe}. Only the last quantity is simultaneously retrospective, sampling-aware, probabilistic, and directly usable as a long-run search prior.

\section{Benchmark Construction Details}
\label{app:benchmark}

We construct the benchmark in Habitat~\citep{savva2019habitat} using ten HM3D-Semantics homes. Each home contains 16--33 valid receptacle locations, including tables, desks, counters, islands, shelves, cabinets, nightstands, dressers, sofas, beds, benches, and stools. Each home receives eight household objects selected from screened Google Scanned Objects~\citep{downs2022gso}, 80 object--home pairs in total, of which three per home serve as calibration objects and five as evaluation queries (below).

\paragraph{Placement distributions.}
Each object is assigned a hidden distribution over valid receptacles, covering unimodal, bimodal, near-uniform, and rare/absent routines. Before every traversal, its location is sampled independently from this distribution and the object is settled using the simulator's physics. Placement is therefore independent across episodes conditioned on the object's long-term distribution.

\paragraph{Support-level ground truth.}
\label{app:support-gt}
For distribution evaluation, we represent each annotated placement location $k$ by a horizontal disc of radius 15\,cm centered at its receptacle-top position. We sample the disc on a 2\,cm grid and assign each sample to the support owning the nearest mapped surface voxel within 10\,cm horizontally and 25\,cm vertically. Let $n_{jk}$ be the number of samples assigned to support $j$ and define $A_{jk}=n_{jk}/\sum_\ell n_{\ell k}$ for a footprint with at least one assigned sample; otherwise its column is zero. Given the query's hidden placement probabilities $\pi_k$, its support mass is $w_j^\ast=\sum_k A_{jk}\pi_k$. This fixed footprint model splits placement mass across patch boundaries in proportion to the observed footprint area. The mean present-placement mass is .850 per query, of which .083 cannot be assigned; normalizing $w_j^\ast$ into $g_j$ conditions on the remaining mass. The same $g_j$ supplies the labels for AP and AUROC and the mass for Top-5.

\paragraph{Presence ground truth.} For the Bernoulli NLL, support $j$ holds placement $k$ ($m_{jk}=1$) when the placement lies within 45\,cm of a voxel of $\mathcal S_j$, the radius at which detections are attributed to supports and at which the calibration labels $Z_{je}$ are defined. The ground-truth prevalence is $\theta_j^\ast=\sum_k\pi_km_{jk}$. A placement can belong to several adjacent supports, so $\sum_j\theta_j^\ast$ averages 4.76 per query, whereas the footprint allocation sums to at most one.

\paragraph{Metric definitions.} Supports are ordered by decreasing score. $\mathrm{AP}=\frac{1}{N_+}\sum_k\mathrm{Prec}@k\,y_{(k)}$ with $N_+=\sum_jy_j$, where supports with equal scores enter the precision--recall curve together (one threshold per distinct score). Top-5 mass is $\sum_jg_j\,\omega_j$, where $\omega_j=1$ for supports ranked in the top five and supports tied at the cutoff receive their expected share. AUROC counts ties as one half. An all-zero score vector ties every support, which gives AP equal to the fraction of positive supports, Top-5 equal to $5/J$ and AUROC .5. For $\mathrm{NLL}_{\mathrm{Ber}}$, $\hat\theta_j$ is clipped to $[10^{-6},1-10^{-6}]$; scores that are not probabilities (detection count, RAVEN- and DynaMem-style) are not scored.

\paragraph{Footprint sensitivity.} With footprint radii of 10\,cm and 20\,cm instead of 15\,cm, \ECROM{} still has the highest AP (.199 and .239), Top-5 (.369 and .360) and AUROC (.767 and .759) of all rows of Table~\ref{tab:sim-results}.

\paragraph{Observation process.}
Observation is generated independently of object placement. Every traversal starts at the dock, moves with discrete simulator actions, and returns to the dock. Each home's 20 traversals mix 3--9 broad routes with 11--17 local patrols (200 traversals, 79,741 RGB-D frames and 9.01\,km in total). On a broad route, rooms are entered with a fixed per-house probability of .25, .5, or .8; within entered rooms, the route covers a random 45--75\% of candidate stations, and 40--70\% of visited stations are traversed without stopping. Camera behavior ranges from broad sweeps to sparse or single glances, with occasional close receptacle inspections from 1.5--3\,m. A local patrol is a short loop of at most 35\,m that visits two of the four stations nearest the dock and scans by turning in place. Consequently, the same support may be fully observed, partially observed, or missed entirely across different traversals.

This independence is deliberate: whether an object is present cannot cause the robot to visit or look toward its location. The benchmark can therefore measure whether a memory correctly separates object prevalence from sensing frequency.

\paragraph{Splits and evaluation protocol.}
Each home contains 20 mapping traversals and three in-distribution test traversals; the benchmark additionally provides two test traversals per home whose placement distributions are shifted, which are released with the benchmark but not used in this paper. Three objects per home are used only to select detector, calibration, and prior parameters; the remaining five are held out, resulting in 50 evaluation queries. Queries are supplied only after the mapping history has been recorded. All memories are built from the 20 mapping traversals. The distribution metrics compare each memory against the known placement distribution of the query. Each of the three test traversals contains one placement of every object sampled from that distribution, and each such placement is one search trial, so every query has exactly three trials and the 50 queries give the 150 trials of Table~\ref{tab:sim-results}.

Every method receives the same RGB-D streams and per-frame poses in a common coordinate frame and may construct its own metric or topological memory; for a fair comparison, our experiments give every memory the same supports and localized detections (Appendix~\ref{app:baselines}). Ground-truth object placements and receptacle annotations are withheld from the memory methods and used only for calibration on the designated objects and for evaluation. Distribution estimation is evaluated on the shared surface supports with footprint-based ground truth (Section~\ref{sec:sim}), while active search uses the same simulator, detector, travel budget, and starting state across methods.

\section{Representation and Implementation Details}
\label{app:impl}

The CROSS keyframe graph~\citep{wang2026cross} supplies graph topology, metric edge costs, and multi-session frame poses. Farthest-point seeding in graph distance followed by multi-source assignment creates connected places; each place has eight yaw bins, and these place--view nodes are the candidate actions and the cost structure of the search. A support is a $0.5\times0.5$\,m patch (per 0.5\,m height band) of 10\,cm up-facing depth voxels observed at least twice within 3\,m, built from the stored keyframe depth in the odometry frame of the traversals; $c_{je}$ is the fraction of the patch's voxels with a depth return in traversal $e$ within the detection range. Detector regions are lifted to 3D and assigned when their point lies within 45\,cm of reliable surface voxels. The score likelihood ratio $\Lambda$ is fitted once on the calibration objects (Appendix~\ref{app:calibration}): 13 histogram bins over the score with Laplace smoothing, pool-adjacent-violators monotonicity and a clipped log ratio $|\log\Lambda|\le6$. The posterior of Eq.~\ref{eq:posterior} is integrated on a composite grid of 400 nodes, 200 log-spaced on $(10^{-9},0.05)$ and 200 uniform on $[0.05,1)$, each node carrying the prior mass of its cell. This matters: the sparse Beta prior has an integrable singularity at zero that a uniform midpoint grid resolves poorly (with $\BetaDist(.05,8)$ the no-evidence posterior mean on a uniform 400-point grid is $.026$ against the exact $.0062$, so every never-observed support would carry about four times too much belief), whereas the composite grid reproduces the exact conjugate Beta--Bernoulli means and intervals. The prior $\BetaDist(\alpha_0,\beta_0)=\BetaDist(.05,8)$ is selected on the calibration objects under this quadrature by held-out presence AP subject to the held-out Bernoulli NLL staying within 5\% of its minimum. The planner's acquisition cost in Eq.~\ref{eq:index} is $d_0=3$\,m, selected on the calibration objects from $\{1,3\}$\,m; travel costs $C_G$ are shortest-path distances in the place graph, and a candidate view is admissible only while the budget still allows reaching it.

\paragraph{Runtime.} Query time is dominated by open-vocabulary detection. All results in the paper score every stored frame with OWLv2, which takes 80\,s per query on an NVIDIA GeForce RTX 5090 (1{,}719 stored frames per home on average, scored at 22.6 frames per second); given the detector evidence, the posterior takes 12\,ms per query and the index planner 0.6\,ms per search step. Detection can be made faster by first retrieving the frames most similar to the query with SigLIP~\citep{zhai2023siglip} and scoring only those with OWLv2. A frame that the detector did not process is then censored: it contributes no evidence, and the opportunity $c_{je}$ of every cell is recomputed from the depth of the processed frames only, so an unprocessed frame is neither a non-detection nor an observation. This trades accuracy for speed. Scoring half of each home's frames halves the query time to 40\,s and lowers AP from .216 to .191 and SR from .587 to .533; scoring a fixed 400 frames takes 19\,s but lowers SR to .480, and about half of this loss comes from the largest home, where the retrieved frames leave most surfaces with little opportunity.

\paragraph{Degenerate cases.} A support with no localized detection in an episode has $r_{je}=0$, never a missing value, and is evaluated by $\Lambda(0)$; a support with $\sum_ec_{je}=0$ keeps its prior mean under \ECROM{} and its binary-exposure and binary-sensor variants, whereas the no-exposure variant sets $c_{je}=1$ for every support in every traversal; an all-zero score vector ties every support in the ranking metrics; the visit-frequency score is zero when $\sum_ev_{je}=0$. During search the $120$\,m budget counts forward motion in $0.2$\,m steps, rotations are free, and the acquisition cost $d_0$ enters only the index.

\subsection{Calibrating the model: likelihood ratio, detection model, prior}
\label{app:calibration}

Besides the opportunity $c_{je}$, which is measured from the stored depth
and needs no calibration, the system uses three quantities that the
derivation does not determine: the score likelihood
ratio $\Lambda(r)$ of Eq.~\ref{eq:Lambda}, the detection probability
$p_{\mathrm{det}}$ of the online belief update
(Eq.~\ref{eq:update}), and the Beta prior $\BetaDist(\alpha_0,\beta_0)$.

Table~\ref{tab:calib-regimes} summarizes the two calibration regimes used
in the paper and where each is used; the ground-truth placements and
receptacle annotations of the evaluation queries enter neither of them.

\begin{table}[t]
\caption{Calibration regimes for the score likelihood ratio. Simulator labels exist only for the designated calibration objects; evaluation queries never receive them.}
\label{tab:calib-regimes}
\centering
\small
\setlength{\tabcolsep}{5pt}

\begin{tabularx}{\linewidth}{>{\raggedright\arraybackslash}p{2.0cm}>{\raggedright\arraybackslash}X>{\raggedright\arraybackslash}p{2.2cm}>{\raggedright\arraybackslash}p{2.6cm}}
\toprule
Setting & Source of $\Lambda$ & Labels & Used for \\
\midrule
Simulation & class histograms of all episode--support cells of the calibration objects, pooled over houses & simulator $Z_{je}$ & every \ECROM{} row of Table~\ref{tab:sim-results} \\
Robot & score-stratified sample of about 120 stored keyframes per labeled prompt, pooled over prompts; histograms shrunk towards the simulation histograms ($n_0=20$) & VLM, whole keyframe & Table~\ref{tab:realrobot} \\
\bottomrule
\end{tabularx}
\end{table}

\paragraph{Likelihood ratio in simulation.}
$\Lambda(r)=p_1(r)/p_0(r)$ is fitted once on the calibration objects, three per home, which are never evaluated. Every episode--support cell of a calibration object with a stored frame or a response contributes its response $r_{je}$, labeled with the simulator's $Z_{je}$: the object is present on support $j$ (the 45\,cm rule of Appendix~\ref{app:support-gt}) and visible with at least 50 pixels within 3\,m in at least one stored keyframe of episode $e$. A cell without a localized detection has $r_{je}=0$ and falls in the lowest score bin, so silence is calibrated together with the responses. The class histograms over 13 score bins with Laplace smoothing give $p_1$ and $p_0$; $\log\Lambda$ is made non-decreasing by pool-adjacent-violators and clipped to $|\log\Lambda|\le6$. The labels are thus the model's latent variable $Z_{je}$, not the correctness of a detector proposal. For the binary-sensor ablation of Section~\ref{sec:ablation}, the same histograms give hit and false-alarm rates $h_\tau=.506$ and $f_\tau=.0087$ at $\tau=.25$, and $h_\tau=.761$ and $f_\tau=.0351$ at $\tau=.15$.

\paragraph{Sensor-model audit.}
Table~\ref{tab:lr-audit} re-estimates $\log\Lambda$ on the calibration objects of LOOP-Bench in five score bins, separately for three ranges of the number of stored keyframes that observed the support in the episode and three ranges of the opportunity $c_{je}$. Scores of .15 and above are positive evidence in every stratum, and their log ratio varies by at most .65 nats across view counts. Weak responses depend more on the number of views: a response in $[.05,.15)$ is neutral when few keyframes observed the support ($+.09$) and negative when many did ($-1.13$), as expected for a maximum over more views of a visible object. Across opportunity levels the ordering of the bins is the same, and a response below .05 is weaker negative evidence at low opportunity ($-1.43$ against $-2.26$).

\begin{table}[t]
\caption{Sensor-model audit on the calibration objects: $\log\Lambda$ per score bin, estimated separately for three ranges of the number of stored keyframes that observed the support in the episode (views) and three ranges of the opportunity $c_{je}$, with half-count smoothing. The last row gives the number of visible cells ($Z_{je}=1$) in each stratum.}
\label{tab:lr-audit}
\centering
\small
\setlength{\tabcolsep}{4.5pt}
\begin{tabular}{lccccccc}
\toprule
& & \multicolumn{3}{c}{Views} & \multicolumn{3}{c}{Opportunity $c_{je}$} \\
\cmidrule(lr){3-5}\cmidrule(lr){6-8}
Score bin & All & $\le11$ & 12--19 & $\ge20$ & $<\frac13$ & $\frac13$--$\frac23$ & $\ge\frac23$ \\
\midrule
$[0,.05)$ & $-2.17$ & $-1.92$ & $-2.30$ & $-2.57$ & $-1.43$ & $-2.55$ & $-2.26$ \\
$[.05,.15)$ & $-0.26$ & $+0.09$ & $-0.34$ & $-1.13$ & $+0.26$ & $-0.39$ & $-0.37$ \\
$[.15,.25)$ & $+2.01$ & $+2.10$ & $+2.02$ & $+1.83$ & $+1.79$ & $+1.95$ & $+2.03$ \\
$[.25,.35)$ & $+3.13$ & $+3.29$ & $+2.99$ & $+2.85$ & $+4.50$ & $+3.28$ & $+2.94$ \\
$[.35,1]$ & $+5.03$ & $+4.76$ & $+4.94$ & $+5.41$ & $+5.07$ & $+4.70$ & $+5.09$ \\
\midrule
Visible cells & 440 & 196 & 134 & 110 & 50 & 74 & 316 \\
\bottomrule
\end{tabular}
\end{table}

\paragraph{Likelihood ratio on the robot.}
The robot has no simulator labels, so a single $\Lambda$ is fitted on its own stored keyframes of the mapping traversals, pooled over the labeled prompts. For each prompt, a score-stratified sample of about 120 keyframes (every keyframe scoring above $.10$, capped at 60, plus a random subset of the low-score keyframes) is labeled \emph{visible} or \emph{not visible} by a vision--language model (OpenAI GPT-6 via the Codex CLI, model \texttt{gpt-6-astra}, default decoding) with the prompt ``This is a photo taken by a robot's camera. Is there \{query\} visible in this image? It may be small or partly occluded, but must be identifiable as \{query\}. Answer with exactly one word: yes or no.'' Each keyframe carries the inverse sampling weight of its score bin, and the weighted class histograms are shrunk towards the simulation histograms with $n_0=20$ pseudo-counts per class before pool-adjacent-violators and clipping. 

\paragraph{Detection model of the online update.}
A live look is summarized by verified discovery, $Y=\ind[r\ge\tau]\cdot\ind[\text{verified}]$ with $\tau=.25$, and its rate given a visible target, $p_{\mathrm{det}}=\Prob(Y{=}1\mid\text{visible})$, is estimated per stored keyframe on the calibration objects of each house: the fraction of keyframes in which the object is visible (at least 50 pixels within the detection range) that would have ended the search, i.e., in which the detector fires and the verification criterion (at least 500 target pixels) holds; it ranges from .05, the floor of the estimate, to .73 across houses. The robot instead uses the raw detector event with the detector rates of its VLM-labeled keyframes (Appendix~\ref{app:realrobot}). A verified discovery cannot occur when the target is not visible, so the update has no false-alarm term. Applying the full
$\Lambda(r)$ of the memory to every live frame is not equivalent:
$\Lambda$ describes the maximum score of a traversal, not a single look.

\paragraph{Beta prior.}
$\alpha_0$ and $\beta_0$ are selected on the calibration objects by
two-fold cross-validation over their mapping traversals, maximizing
held-out presence AP subject to the held-out Bernoulli NLL staying within
5\% of its minimum, which gives $\BetaDist(.05,8)$ (prior mean about
$.0062$: a named object occupies a given surface patch in roughly one
traversal in 161). The exact values matter little. The prior is weak, so
the evidence of repeatedly observed supports dominates it, and it mainly
sets the belief of supports with little or no evidence, which stay at the
prior mean. What it must encode is only that a particular surface rarely
holds a particular object, $\alpha_0\ll\beta_0$. This is why the robot,
which has no presence labels to select the prior, can keep the
benchmark's values.

\section{Paired Uncertainty, Present-Target Search and Per-House Results}
\label{app:uncertainty}

Table~\ref{tab:paired} gives the paired difference between \ECROM{} and every other row of Table~\ref{tab:sim-results}, with 95\% intervals from a two-level bootstrap that resamples houses and then queries within each resampled house (20{,}000 resamples), so the dependence between queries of a house and between the three trials of a query is preserved. Table~\ref{tab:present} reports SR and SPL over the 128 present-target trials and the mean distance traveled on the 22 absent-target trials. On absent-target trials, detection count, visit frequency, the opportunity-normalized count and RAVEN-style retrieval stop after 58--81\,m because their scores are zero on most supports and their candidates run out; all other methods use the whole 120\,m budget. Table~\ref{tab:perhouse} gives the present-target SR of every house, and Table~\ref{tab:mechanism} splits the 50 queries into thirds by the opportunity summed over their true supports and weighted by occurrence probability, $\sum_j g_j\sum_e c_{je}$ (cut points 2.27 and 3.92).

\begin{table}[h]
\caption{Paired differences, \ECROM{} minus each method, with 95\% house-then-query bootstrap intervals. Positive values favor \ECROM{} except for $\mathrm{NLL}_{\mathrm{Ber}}$, where negative values do.}
\label{tab:paired}
\centering
\scriptsize
\setlength{\tabcolsep}{3pt}
\begin{tabular}{lccc}
\toprule
Method & AP & Top-5 & AUROC \\
\midrule
Detection count & +.068 [+.034, +.105] & +.050 [+.016, +.089] & +.101 [+.053, +.151] \\
Visit frequency & +.101 [+.058, +.147] & +.106 [+.044, +.182] & +.141 [+.090, +.193] \\
Opportunity-normalized count & +.045 [+.005, +.092] & +.057 [+.001, +.123] & +.105 [+.059, +.154] \\
PredictiveGraphs-style & +.056 [+.017, +.095] & +.105 [+.012, +.196] & +.101 [+.052, +.151] \\
RAVEN-style retrieval & +.067 [+.031, +.105] & +.027 [$-$.054, +.109] & +.160 [+.113, +.205] \\
DynaMem-style latest state & +.115 [+.060, +.174] & +.171 [+.085, +.270] & +.132 [+.067, +.199] \\
Coverage (no memory) & +.205 [+.144, +.269] & +.361 [+.264, +.463] & +.260 [+.212, +.303] \\
\midrule
Without exposure term ($c_{je}{=}1$) & +.025 [$-$.005, +.060] & +.026 [$-$.016, +.080] & +.042 [$-$.017, +.100] \\
Binary exposure ($\ind[c_{je}{>}0]$) & +.021 [$-$.001, +.045] & +.014 [$-$.003, +.037] & +.027 [+.006, +.051] \\
Binary sensor ($\tau{=}.25$) & +.021 [$-$.006, +.050] & +.022 [$-$.019, +.073] & +.014 [$-$.001, +.034] \\
Binary sensor ($\tau{=}.15$) & +.041 [+.017, +.069] & +.049 [$-$.001, +.112] & +.008 [$-$.004, +.020] \\
\bottomrule
\end{tabular}

\vspace{4pt}
\begin{tabular}{lccc}
\toprule
Method & $\mathrm{NLL}_{\mathrm{Ber}}$ & SR & SPL \\
\midrule
Detection count & -- & +.147 [+.047, +.260] & +.051 [$-$.003, +.108] \\
Visit frequency & $-$.0524 [$-$.0724, $-$.0365] & +.153 [+.060, +.267] & +.083 [+.024, +.150] \\
Opportunity-normalized count & $-$.0486 [$-$.0731, $-$.0291] & +.120 [+.027, +.227] & +.071 [+.003, +.149] \\
PredictiveGraphs-style & $-$.0015 [$-$.0050, +.0011] & +.053 [$-$.013, +.127] & +.043 [$-$.037, +.117] \\
RAVEN-style retrieval & -- & +.167 [+.067, +.280] & +.071 [+.006, +.143] \\
DynaMem-style latest state & -- & +.187 [+.087, +.307] & +.115 [+.036, +.208] \\
Coverage (no memory) & $-$.0070 [$-$.0103, $-$.0045] & +.340 [+.193, +.487] & +.308 [+.200, +.422] \\
\midrule
Without exposure term ($c_{je}{=}1$) & $-$.0023 [$-$.0038, $-$.0010] & +.067 [$-$.040, +.167] & +.049 [$-$.009, +.098] \\
Binary exposure ($\ind[c_{je}{>}0]$) & +.0000 [$-$.0002, +.0002] & +.047 [$-$.027, +.133] & +.021 [$-$.021, +.064] \\
Binary sensor ($\tau{=}.25$) & $-$.0010 [$-$.0015, $-$.0004] & +.113 [+.040, +.193] & +.043 [$-$.007, +.092] \\
Binary sensor ($\tau{=}.15$) & $-$.0022 [$-$.0031, $-$.0015] & +.047 [$-$.027, +.113] & +.079 [+.012, +.143] \\
\bottomrule
\end{tabular}
\end{table}

\begin{table}[h]
\begin{minipage}[t]{0.52\linewidth}
\centering
\captionof{table}{Search over the 128 present-target trials and mean travel (m) on the 22 absent-target trials.}
\label{tab:present}
\scriptsize
\setlength{\tabcolsep}{3pt}
\begin{tabular}{lccc}
\toprule
Method & SR & SPL & Absent travel \\
\midrule
\ECROM{} (ours) & .688 & .447 & 120.0 \\
Detection count & .516 & .387 & 61.2 \\
Visit frequency & .508 & .350 & 57.7 \\
Opportunity-normalized count & .547 & .364 & 80.5 \\
PredictiveGraphs-style & .625 & .397 & 120.0 \\
RAVEN-style retrieval & .492 & .363 & 71.6 \\
DynaMem-style latest state & .469 & .313 & 120.0 \\
Coverage (no memory) & .289 & .086 & 120.0 \\
Without exposure term ($c_{je}{=}1$) & .609 & .390 & 120.0 \\
Binary exposure ($\ind[c_{je}{>}0]$) & .633 & .422 & 120.0 \\
Binary sensor ($\tau{=}.25$) & .555 & .397 & 120.0 \\
Binary sensor ($\tau{=}.15$) & .633 & .355 & 120.0 \\
\bottomrule
\end{tabular}
\end{minipage}\hfill
\begin{minipage}[t]{0.45\linewidth}
\centering
\captionof{table}{Present-target SR per house: \ECROM{}, without exposure term ($c{=}1$), binary exposure ($\ind[c{>}0]$) and PredictiveGraphs-style (PG); $n$ present trials.}
\label{tab:perhouse}
\scriptsize
\setlength{\tabcolsep}{3pt}
\begin{tabular}{lccccc}
\toprule
House & \ECROM{} & $c{=}1$ & $\ind[c{>}0]$ & PG & $n$ \\
\midrule
2Pc8 & .929 & .643 & .857 & 1.000 & 14 \\
5biL & .917 & .917 & .917 & .917 & 12 \\
7MXm & .615 & .308 & .462 & .385 & 13 \\
DqJK & .308 & .615 & .385 & .308 & 13 \\
HfMo & .467 & .400 & .400 & .400 & 15 \\
iePH & .923 & .769 & 1.000 & .846 & 13 \\
JptJ & .750 & .750 & .750 & .750 & 12 \\
L5QE & .583 & .417 & .583 & .500 & 12 \\
q5QZ & .727 & .636 & .636 & .636 & 11 \\
XiJh & .692 & .692 & .385 & .538 & 13 \\
\bottomrule
\end{tabular}
\end{minipage}
\end{table}

\begin{table}[h]
\caption{Gain of the exposure term by how well the true supports were observed: \ECROM{} minus the variant without exposure term ($c_{je}{=}1$), averaged over the queries of each third.}
\label{tab:mechanism}
\centering
\small
\begin{tabular}{lccccc}
\toprule
Queries & $n$ & $\Delta$AP & $\Delta$Top-5 & $\Delta$AUROC & $\Delta\mathrm{NLL}_{\mathrm{Ber}}$ \\
\midrule
Least observed third & 17 & +.057 & +.052 & +.112 & $-$.0032 \\
Middle third & 16 & +.020 & +.013 & +.038 & $-$.0031 \\
Best observed third & 17 & $-$.001 & +.011 & $-$.024 & $-$.0005 \\
\bottomrule
\end{tabular}
\end{table}

\section{Real-Robot Experiment: Per-Query Results}
\label{app:realrobot}

Table~\ref{tab:realrobot-perquery} lists the outcome of every real-robot query of Section~\ref{sec:realrobot}; Table~\ref{tab:realrobot} in the main text aggregates these 15 trials per method. On the robot, the belief is updated with the detector event $\ind[r\ge\tau]$ described below instead of the verified discovery of Eq.~\ref{eq:update}; travel of a failed trial is the distance traveled when its budget was exhausted.

\paragraph{Belief update on the robot.} The robot plans with the index rule of Eq.~\ref{eq:index}. After each scored frame, the belief of the single support $j$ the frame is assigned to (from its camera pose) is multiplied by
\begin{equation}
\Lambda_D=\begin{cases}p_{\mathrm{det}}/p_{\mathrm{false}}, & r\ge\tau,\\ (1-p_{\mathrm{det}})/(1-p_{\mathrm{false}}), & r<\tau,\end{cases}
\end{equation}
all other supports keep their belief, unlike the exposure-weighted update of Eq.~\ref{eq:update}, and the belief is renormalized. $p_{\mathrm{det}}$ and $p_{\mathrm{false}}$ are the weighted rates $\Prob(r\ge\tau\mid\text{visible})$ and $\Prob(r\ge\tau\mid\text{not visible})$ of the VLM-labeled keyframes. A trial ends when the detector fires at an inspected view and the operator confirms the target; the operator rejected no trigger in the 45 trials. Table~\ref{tab:sim-vs-robot} lists every difference between the simulated and the robot system.

\begin{table}[h]
\caption{Differences between the simulated system of Table~\ref{tab:sim-results} and the robot system of Table~\ref{tab:realrobot}.}
\label{tab:sim-vs-robot}
\centering
\small
\setlength{\tabcolsep}{4pt}
\begin{tabularx}{\linewidth}{>{\raggedright\arraybackslash}p{3.0cm}>{\raggedright\arraybackslash}X>{\raggedright\arraybackslash}X}
\toprule
Component & Simulation & Robot \\
\midrule
Mapping traversals & 20 per house & 10 \\
Likelihood ratio $\Lambda$ & calibration objects, simulator labels $Z_{je}$ & VLM labels of stored keyframes (whole-image question), shrunk towards the simulation histograms \\
Live belief update & verified discovery, Eq.~\ref{eq:update}, on every support weighted by its exposure $c_j(u)$ & detector event $\ind[r\ge\tau]$ with $p_{\mathrm{det}}$ and $p_{\mathrm{false}}$, on the one support each frame is assigned to \\
Success & detector fires and at least 500 target pixels are in the executed view & detector fires and the operator confirms \\
Budget & 120\,m of travel & 20 inspected views \\
\bottomrule
\end{tabularx}
\end{table}

\begin{table}[h]
\caption{Real-robot search on the Stretch~3 per query: successful trials out of three and mean travel in meters for each memory (mean over the three trials, failed trials included at the distance traveled before stopping). The last two rows aggregate over all 15 trials. Bold marks the best method.}
\label{tab:realrobot-perquery}
\centering
\small
\setlength{\tabcolsep}{5pt}
\renewcommand{\arraystretch}{1.08}
\begin{tabular}{lcc@{\hspace{12pt}}cc@{\hspace{12pt}}cc}
\toprule
& \multicolumn{2}{c}{\ECROM{} (ours)} & \multicolumn{2}{c}{Detection count} & \multicolumn{2}{c}{PredictiveGraphs-style} \\
\cmidrule(lr){2-3}\cmidrule(lr){4-5}\cmidrule(lr){6-7}
Query & Success & Travel (m) & Success & Travel (m) & Success & Travel (m) \\
\midrule
red cup & 3/3 & 25.4 & 2/3 & 49.9 & 2/3 & 42.2 \\
luncheon meat & 2/3 & 17.7 & 0/3 & 17.1 & 1/3 & 29.7 \\
headphones & 3/3 & 24.7 & 3/3 & 26.8 & 3/3 & 26.9 \\
cupcakes & 3/3 & 20.0 & 3/3 & 21.5 & 3/3 & 23.1 \\
scissors & 3/3 & 9.0 & 2/3 & 30.4 & 2/3 & 25.6 \\
\midrule
All trials (SR, travel) & \textbf{14/15 (.93)} & \textbf{19.4} & 10/15 (.67) & 29.1 & 11/15 (.73) & 29.5 \\
Views inspected & \multicolumn{2}{c}{\textbf{4.7}} & \multicolumn{2}{c}{9.3} & \multicolumn{2}{c}{8.9} \\
\bottomrule
\end{tabular}
\end{table}

\end{document}